%% file: main.tex
\documentclass{article}
\usepackage[final, nonatbib]{neurips_2022}
\usepackage[round]{natbib}
\usepackage[utf8]{inputenc}        
\usepackage[T1]{fontenc}           
\usepackage{hyperref}              
\usepackage{url}                   
\usepackage{booktabs}              
\usepackage{amsfonts}              
\usepackage{nicefrac}              
\usepackage{microtype}             
\usepackage{amsmath}               
\usepackage{mathtools}             
\usepackage[capitalise]{cleveref}  
\usepackage{graphicx}              
\usepackage{amsthm}                
\usepackage{tcolorbox}             
\usepackage{algorithm}             
\usepackage{algcompatible}         
\usepackage{algpseudocode}         

\newtheorem{theorem}{Theorem}
\newtheorem{lemma}{Lemma}
\newtheorem{definition}{Definition}

\tcolorboxenvironment{theorem}{
    boxrule=0pt, boxsep=0pt, colback=gray!10, arc=0pt, outer arc=0pt
}

\tcolorboxenvironment{lemma}{
    boxrule=1pt, boxsep=0pt, colback=white, arc=0pt, colframe=black
}
\tcolorboxenvironment{definition}{
    boxrule=1pt, boxsep=0pt, colback=white, arc=0pt, colframe=black
}

\input{shortcuts}

\title{Primal-Dual Sample Complexity Bounds for Constrained Markov Decision Processes with \\ Multiple Constraints}

\author{
    \begin{tabular}{c@{\hspace{1cm}}c@{\hspace{1cm}}c}
        Max Buckley \thefootnote{*} & Konstantinos Papathanasiou \thefootnote{*} & Andreas Spanopoulos \thefootnote{*} \\
        {\normalfont Google} & {\normalfont ETH Zurich} & {\normalfont ETH Zurich}\\
        \texttt{buckleym@google.com} & \texttt{kpapathanasi@ethz.ch} & \texttt{aspanopoulos@ethz.ch}
    \end{tabular}
}

\begin{document}
\maketitle
\def\thefootnote{*}\footnotetext{The authors are listed in alphabetical order.}
\begin{abstract}
    This paper addresses the challenge of solving Constrained Markov Decision Processes (CMDPs) with $d > 1$ constraints when the transition dynamics are unknown, but samples can be drawn from a generative model. We propose a model-based algorithm for infinite horizon CMDPs with multiple constraints in the tabular setting, aiming to derive and prove sample complexity bounds for learning near-optimal policies. Our approach tackles both the relaxed and strict feasibility settings, where relaxed feasibility allows some constraint violations, and strict feasibility requires adherence to all constraints. The main contributions include the development of the algorithm and the derivation of sample complexity bounds for both settings. For the relaxed feasibility setting we show that our algorithm requires $\cOT \left( \frac{d |\cS| |\cA| \log(1/\delta)}{(1-\gamma)^3\epsilon^2} \right)$ samples to return $\epsilon$-optimal policy, while in the strict feasibility setting it requires $\cOT \left( \frac{d^3 |\cS| |\cA| \log(1/\delta)}{(1-\gamma)^5\epsilon^2{\slatercs}^2} \right)$ samples.
\end{abstract}

\section{Introduction}

Reinforcement Learning (RL) is a dynamic area within machine learning that focuses on agents learning which actions to take in order to maximize their cumulative reward. This field has seen substantial success in diverse domains such as robotics, autonomous driving, healthcare, and game playing \citep{sutton2018reinforcement, silver2016mastering}. The standard framework for RL problems is the Markov Decision Process (MDP), where an agent makes decisions by considering the current state, selecting actions, and transitioning to new states according to certain probabilities, all while accumulating rewards.

However, many real-world problems impose constraints on the agent's behavior. The prevalent framework used to model such cases is Constrained Markov Decision Processes (CMDPs) \citep{altman1999constrained}. CMDPs extend MDPs by introducing the concept of a cost. The constraints then restrict the expected costs of admissible policies. These constraints can be used to reflect a range of potential limitations like: unsafe regions in robot navigation \citep{feyzabadi2015hcmdp}, budget allocation in finance \citep{xiao2019model}, or safety in autonomous systems \citep{garcia2015comprehensive}. These constraints are crucial for ensuring feasible and safe operation in practical applications.

Traditional approaches to solving CMDPs rely on knowledge of the transition dynamics, which is often unrealistic in complex environments. However, it is often the case that we instead have access to a generative model. By querying such a model with a given state-action pair, we can sample the next state. This allows us to follow a model-based approach to solving this problem by constructing an empirical model of the transition dynamics. Such generative oracles are practical in many applications, where simulations or real-world data collection can provide the necessary samples required to approximate the environment's behavior without requiring explicit knowledge of the underlying dynamics.

In this work, we focus on solving CMDPs in scenarios where the transition dynamics are a priori unknown, but a generative model is available. Our aim is to derive and prove sample complexity bounds for learning near-optimal policies under both relaxed and strict feasibility settings. The relaxed feasibility setting allows for some constraint violations, whereas the strict feasibility setting requires policies to strictly adhere to all constraints, which is critical in high-stakes applications such as healthcare or autonomous driving where constraint violations can be disastrous.

The main contributions of this research are as follows:
\begin{itemize}
    \item We develop an algorithm for solving infinite horizon CMDPs with multiple constraints in the tabular setting, when the transition dynamics are unknown but we have access to a generative model, addressing both relaxed and strict feasibility settings.
    \item We derive sample complexity bounds for our algorithm, providing theoretical guarantees on the number of samples needed to achieve near-optimal policies.
    \item Our analysis covers both the relaxed and strict feasibility settings, showing how the sample complexity differs between the two and underlining the challenges of achieving the latter.
\end{itemize}


\section{Related Work}

CMDPs can be modeled utilizing a model-free or a model-based approach. Under the model-free setting, \citep{ding2020natural} proposed a primal-dual natural policy gradient algorithm (NPG-PD) operating and providing non-asymptotic convergence guarantees in both tabular and general settings. \citep{xu2021crpo} proposed the constraint-rectified policy optimization (CRPO) algorithm which adopts a primal update and provides global optimality guarantee. Assuming zero constraint violations, \citep{wei2021provably} proposed the Triple-Q algorithm while \citep{xu2021crpo} introduced a conservative stochastic primal-dual algorithm (CSPDA) and provided the corresponding sample complexity bounds.

Regarding the model-based approach, the CMDP is solved using a model-based algorithm after the transition matrix is provided or estimated with sufficient accuracy. Under this framework \citep{efroni2020exploration} proposed different algorithms that leverage the primal, dual or primal-dual formulation of the CDMP and achieve sublinear objective and constrained regret. \citep{brantley2020constrained} proposed CONRL algorithm which relies on the principle of optimism under uncertainty operating under convex-concave and knapsack settings. \citep{ding2021provably} proposed the OPDOP algorithm considering linear transition kernels while they extended their results to the tabular setting as well.

Closer to our approach is the line or research that studies the  asymptotic convergence of their algorithm under the assumption of the unknown model \citep{tessler2018reward, paternain2019constrained}. These studies use the Lagrangian method to show zero duality gap asymptotically. \citep{vaswani2022near} analyzed theoretically the algorithm proposed by \citep{paternain2019constrained} for both the relaxed and the strict feasibility case, considering access to a generative model and a single constraint providing matching complexity bounds $\cOT\left( \frac{|\cS| |\cA| \log(1/\delta)}{(1-\gamma)^3\epsilon^2} \right)$ to the unconstrained case for the former case. \citep{hasanzadezonuzy2021model} also analyzed the relaxed feasibility case and proposed GM-CRL, an extended linear programming algorithm that achieves $\cOT\left( \frac{\gamma^2 |\cS|^2 |\cA| \log(d/\delta)}{(1-\gamma)^3\epsilon^2} \right)$ sample complexity. Finally, \cite{bai2022achieving} examined the strict feasibility case and proposed a conservative stochastic primal-dual algorithm (CSPDA) which achieves $\cOT\left( \frac{d |\cS| |\cA|}{(1-\gamma)^6 \epsilon^2 \zeta^2} \right)$ sample complexity.

\section{Problem Formulation}

\subsection{CMDP Definition}
We consider an infinite-horizon discounted constrained Markov decision process (CMDP) with $d > 1$ constraints defined by the tuple $M^d = \left\langle \cS, \cA, \cP, r, \{c\}_{i=1}^d, \{b\}_{i=1}^d, \rho, \gamma \right\rangle$. Here, $\cS$ represents the state space, $\cA$ represents the action space, $\cP\colon \cS \times \cA \to \Delta_{\cS}$ is the state transition matrix, $\rho \in \Delta_{\cS}$ is the initial distribution of the states and $\gamma \in [0,1)$ is the discount factor. The reward function is denoted by $r\colon \cS \times \cA \to [0,1]$, and the constraint value functions are given by $c_i\colon \cS \times \cA \to [0,1]$, with $b_i \in \R$ representing the required minimum feasibility values for each constraint $i$, where $i = 1, \dots, d$. The value function of a reward objective $l$ for a given stationary policy $\pi \colon \cS \rightarrow \Delta_{\cA}$ is defined as $V_l^\pi(\rho) \coloneq \E_{s_0, a_0, \dots} \left[ \sum_{t = 0}^\infty \gamma^t l(s_t, a_t) \,\mid\, s_0 \sim \rho, a_t \sim \pi(\cdot | s_t), s_{t+1} \sim \cP(\cdot | s_t, a_t) \right]$. Also, for each state-action pair $(s,a)$ and policy $\pi$, the action-value function of a reward objective $l$ is defined as $Q_{l} \colon \cS \times \cA \to \mathbb{R}$ and satisfies: $V_{l}^{\pi}(s) = \langle \pi(\cdot|s), Q_{l}^{\pi}(s,\cdot) \rangle$. The goal of the agent in a CMDP is to find a policy that maximizes the value function associated with the reward function $r$, while ensuring that the policy does not violate any of the constraints. Formally, this objective is expressed as:
\begin{equation}
    \max_\pi\, \Vrpi \quad \mathrm{s.t.} \quad \Vcipi \geq b_i, \;\; \fac \label{eq:CMDP-formulation}
\end{equation}


\subsection{Relaxed \& Strict Feasibility Settings}

When attempting to solve CMDPs under unknown transition dynamics, we usually distinguish between the two settings: relaxed feasibility and strict feasibility. These settings define how strictly the constraints must be adhered to in the optimization process.

In the \textbf{relaxed feasibility} setting, small violations in the reward and constraint value functions are allowed. Formally, we require a policy $\hpi$ such that for some $\epsilon > 0$:

\begin{equation}
    \Vrhpi \geq V_r^*(\rho) - \epsilon, \quad \Vcihpi \geq b_i - \epsilon \quad \fac \label{eq:relaxed-CMDP-formulation}
\end{equation}

In the \textbf{strict feasibility} setting, the constraints must be strictly adhered to, with only small violations allowed in the reward function. For some $\epsilon > 0$, the requirement is:

\begin{equation}
    \Vrhpi \geq V_r^*(\rho) - \epsilon, \quad \Vcihpi \geq b_i \hspace*{9.2mm} \fac \label{eq:strict-CMDP-formulation}
\end{equation}

\subsection{Sample Complexity}
In our setting, we consider the transition matrix $\cP$ to be unknown and in need of estimation. To that end, the agent has access to a generative model which allows for sampling the next state $s_{t+1}$ given any state-action pair $(s_t, a_t)$. This means that for any queried state $s_t \in \cS$ and action $a_t \in \cA$, the generative model provides a sample from the transition probability distribution $\cP(\cdot | s_t, a_t)$. By repeatedly querying this model, the agent can collect sufficient samples to construct an empirical approximation of the CMDP, which it then uses to learn a near-optimal policy.

The primary objective is to construct an algorithm that solves the problems described in \cref{eq:relaxed-CMDP-formulation} and \cref{eq:strict-CMDP-formulation}, and determine the minimum number of samples $N$ required to ensure that for a given $\delta > 0$, the constraints will be satisfied with probability $1 - \cO(\delta)$. We aim to establish bounds on $N$ that guarantee these constraints are met, thereby quantifying the sample complexity needed to achieve near-optimal performance under these feasibility constraints.

\section{Methodology}

Following the approach of \citep{vaswani2022near} which we extend for multiple constraints, we construct an empirical CMDP $\hmdpd = \left\langle \cS, \cA, \hat{\cP}, r_p, \{c\}_{i=1}^d, \{b_i'\}_{i=1}^d, \rho, \gamma \right\rangle$. The primary difference between the actual CMDP $\mdpd$ and the empirical CMDP $\hmdpd$ is the use of the empirical transition matrix $\hat{\cP}$, which is estimated by drawing samples from the generative model. Additionally, we introduce a perturbed reward function $r_p(s, a) \coloneq r(s, a) + \xi(s, a)$, where $\xi(s, a) \sim \mathcal{U}[0, \omega]$, and modify the feasibility constraint $b'$ to distinguish between relaxed ($b' < b$) and strict ($b' = b$) feasibility settings. We specify $\omega, b_i'$ and other parameters later when we instantiate our algorithm.

We denote by $\hV_l^\pi(\rho)$ the value function of an objective $l$ when following policy $\pi$, when using the empirical transition matrix $\hat{\cP}$. The optimization objective for this empirical CMDP $\hmdpd$ is:

\begin{equation}
    \hpis \in \left\{ \argmax_\pi\, \hVrppi \quad \mathrm{s.t.} \quad \hVcipi \,\geq\, b_i',\;\; \fac \right\} \label{eq:emprical-CMDP-formulation}
\end{equation}

\subsection{Saddle-point Formulation}

We can rewrite the empirical CMDP optimization problem in \cref{eq:emprical-CMDP-formulation} as an equivalent saddle-point problem:

\begin{equation}
    \max_\pi\, \min_{\lambda_i \geq 0}\, \left[ \hVrppi + \sum_{i=1}^d \li \left( \hVcipi - b_i' \right) \right] \,\Leftrightarrow\, \max_\pi\, \min_{\li \geq 0}\, \left[ \hVrppi + \bltop \left( \hVcpi - \bb' \right) \right] \label{eq:saddle-point}
\end{equation}

The solution to this saddle point problem is $(\hpis,\, \bls)$, where $\hpis$ is the optimal empirical policy for and $\bls = [\lambda_1^*, \dots, \lambda_d^*]$ is the vector containing the optimal Lagrange multipliers.

It was proven in \citep{paternain2019constrained} that \textbf{this formulation exhibits zero duality gap}. That is, if we consider the reward function of the unconstrained MDP $\cL(\bl, \pi) \coloneq \hVrppi + \bltop \left( \hVcpi - \bb' \right)$, then
\begin{equation}
    \max_\pi \min_{\li \geq 0} \cL(\bl, \pi) \,=\, \min_{\li \geq 0} \max_{\pi} \cL(\bl, \pi) \label{eq:strong-duality}
\end{equation}

\subsection{Solving the saddle-point problem}

As in \citep{vaswani2022near}, we solve the above saddle-point problem iteratively, by alternatively updating the policy (primal variables) and the Lagrange multipliers (dual variables). If $T$ is the total number of iterations of the primal-dual algorithm, we define $\hpit$ and $\blt$ to be the primal and dual iterates for $t \in [T] := {1, \dots, T }$. The primal update at iteration $t$ is given as:
\begin{equation}
    \hpit = \argmax_\pi \left[ \hVrppi + \bltop \hVcpi \right] \coloneq \argmax_\pi \hV_{r_p + \bltop \cc}^\pi(\rho) \label{eq:primal-update}
\end{equation}

The dual update at iteration $t$ is given as:
\begin{equation}
    \bltt = \cR_\Lambda \left[ \ProjP_{\cU} \left[ \blt - \eta \left( \hVcpi - \bb' \right) \right] \right] \label{eq:dual-update}
\end{equation}

where $\eta$ is the step-size for the gradient descent update, $\cU \coloneq [0, U]^d$ and $\ProjP_{\cU} (\xx)$ projects each component of $\xx$ in $[0, U]$, and $\cR_\Lambda (\xx)$ rounds each component of $\xx$ to the closest element in the $\enet$-net $\Lambda = \{0, \enet, 2\enet, \dots, U\}$. We also assume that $U > \blsnorm$ is an upper bound for the largest optimal Lagrange multiplier.

\subsection{Our approach}

To solve this problem, we begin by estimating the empirical transition matrix $\hat{\cP}$ by querying the next state $s'$ for every state-action pair $(s, a)$ $N$ (to be specified later) times in total. That is, we set $\hat{\cP}(s' | s, a) \coloneq \frac{N(s' | s, a)}{N}$, where $N(s' | s, a)$ is the number of samples that have transitions from $(s, a)$ to $s'$. We then set $\omega, U, \enet, \eta, \bb'$ depending on the feasibility setting (relaxed vs strict), and apply the updates in \cref{eq:primal-update} and \cref{eq:dual-update} for $T$ iterations. As we will see in the next sections, by setting these parameters accordingly, we can guarantee that our algorithm will solve \cref{eq:relaxed-CMDP-formulation} and \cref{eq:strict-CMDP-formulation} with $N$ total queries and with probabilities $1 - 3\delta$ and $1 - 4\delta$ respectively.

The introduction of a perturbed reward function $r_p$ in the empirical CMDP, the magnitude of the perturbation $\omega$, the discretization of the Lagrange multipliers using the $\enet$-net $\Lambda$, and the upper-bound on the optimal Lagrange multipliers $U$ are all technicalities that help us prove sample complexity bounds for $N$. While the interested reader can find more details in the proofs in the appendices, we will briefly outline the purpose of these technicalities below so that the motivation behind the theorems can be understood with more ease.

\begin{itemize}
    \item The \textbf{perturbed reward function} $r_p$ is introduced to ensure a separation (or gap, see \cref{def:iota-gap}) in the empirical $\hat{Q}$-values). This separation is then used to apply Lemma 5 of \citep{vaswani2022near}, which is then used to prove concentration bounds between the true and empirical CMDPs for the data-dependent policies $\hpit$. That is, we will later require to bound $\| \Vrphpit - \hVrphpit \|_\infty$.
    \item The \textbf{magnitude of the uniform perturbation} $\omega$ has to be $\leq 1$, and will be chosen such that it simplifies reward suboptimality terms.
    \item The $\enet$-net $\Lambda$ is used to ensure that there will be a finite number of unconstrained MDPs, as each unconstrained MDP is uniquely defined by the Lagrange multipliers $\blt$ (assuming other variables are fixed). This discretization of $\blt$ will later allow us
    to show separation for all possible unconstrained MDPs, by taking a union bound over their finite population.
    \item The \textbf{upper-bound} $U$ on $\blsnorm$ helps in the discretization of $\blt$ and allows us to instantiate our algorithm.
\end{itemize}

\section{Algorithm}
To solve the empirical CMDP $\hmdpd$, inspired by \cite{vaswani2022near} we propose the model-based dual descent algorithm \cref{alg:alg1}. 

\begin{algorithm}

\caption{Model-based algorithm for CMDPs with generative model}\label{alg:alg1}

\textbf{Input:} $\cS$ (State Space), $\cA$ (Action Space), $r$ (rewards), $\mathbf{c}$ (constraint rewards), $N$ (number of samples), $\mathbf{b'}$ (constraint RHS), $\omega$ (perturbation magnitude), $U$ (projection upper bound), $\enet$ (epsilon-net resolution), $T$ (number of iterations), $\boldsymbol{\lambda}_0 = 0$,  (initialization). \\

For each state-action $(s, a)$ pair, collect $N$ samples from $\cP(\cdot |s, a)$ and form $\hat{\cP}$. \\
Perturb the rewards to form vectors $r_p(s, a) \coloneq r(s,a) + \xi(s,a),\xi(s,a) \sim \cU[0, \omega]$. {\scriptsize \Comment{$\omega$ is used here} } \\
Form the empirical CMDP $\hmdpd = \left\langle \cS, \cA, \hat{P}, r_p, \{c\}_{i=1}^d, \{b_i'\}_{i=1}^d, \rho, \gamma \right\rangle$. \\
Form the epsilon-net $\Lambda = \{0, \enet, 2\enet, \cdots, U \}$. {\scriptsize \Comment{$\enet$ and $U$ are used here} } \\
\textbf{for} $t \leftarrow 0$ \textbf{to} $T-1$ \textbf{do}\\
    \hspace*{\algorithmicindent} Update the policy by solving an unconstrained MDP: $\hpit = \argmax_\pi \hV_{r_p + \blttop \cc}^\pi (\rho)$ \\
    \hspace*{\algorithmicindent} Update the dual-variables: $\bltt = \cR_\Lambda \left[ \ProjP_{\cU} \left[ \blt - \eta \left( \hVcpi - \bb' \right) \right] \right]$ {\scriptsize \Comment{$\enet$ and $U$ are used here} } \\ 
\textbf{end} \\

\textbf{Output:} Mixture policy $\bpi_T = \frac{1}{T}\sum_{t=0}^{T-1} \hat{\pi}_t$ {\scriptsize \Comment{To instantiate $\bpit$ we can sample uniformly from $\hpit, t \in [T]$} }

\end{algorithm}
Next we prove (see \cref{Appendix-A}) in  \cref{th:theorem-1} that using the mixture policy returned by \cref{alg:alg1} we can bound the average optimality gap in the reward value function and constraint violation, ensuring that Algorithm \ref{alg:alg1} can solve the empirical CMDP $\hmdpd$.

\begin{theorem}[Guarantees for the primal-dual algorithm]\label{th:theorem-1}
    For a target error $\eopt > 0$ and the primal-dual updates in Eq. \ref{eq:primal-update} - Eq. \ref{eq:dual-update} with $U > \blsnorm$, $T = \frac{4 U^2 d^2}{\eopt^2 (1-\gamma)^2} \left[ 1 + \frac{1}{(U - \blsnorm)^2} \right]$, $\eta =\frac{U (1-\gamma)}{\sqrt{T}}$, and $\enet = \frac{\eopt^2 (1-\gamma)^2 (U - \blsnorm)}{6dU}$, the mixture policy $\bpit \coloneq \frac{1}{T}\sum_{t=0}^{T-1} \hpit$ satisfies
    \[
        \hVrpbpi \geq \hVs - \eopt \quad \text{and} \quad \hVcibpi \geq b_i' - \eopt \, \fac
    \]
    
    where $\hVrpbpi \,=\, \frac{1}{T} \sum_{t = 0}^{T - 1} \hVrphpit$ is the value function of the mixture policy $\bpit$.
\end{theorem}

The result of \cref{th:theorem-1} shows that with $T = \mathcal{O}(\frac{1}{\eopt^2})$ and $\epsilon_1 = \mathcal{O}(\eopt^2)$ the reward achieved by the mixture policy $\bpi$ is $\eopt$ close to the the one achieved by the empirical optimal policy $\hpis$, with a constraint violation of at most $\eopt$. Therefore we can  to solve \cref{eq:emprical-CMDP-formulation} utilizing the primal dual algorithm for a sufficient number of iterations $T$ and and a sufficient small resolution $\epsilon_1$.

\section{Concentration Proofs}
In this section we will provide an overview of the techniques that we utilized in order to bound the concentration terms in \cref{eq:relaxed-decomposition} and \cref{eq:strict-decomposition} respectively in order to prove Theorem \cref{th:theorem-3} and Theorem \cref{th:theorem-4}. To that end we will use an unconstrained MDP  $\mdpd_f = (\cS, \cA, \cP, \gamma, f)$ which has the same state-action space, transition probabilities and discount factor with the CMDP in \cref{eq:CMDP-formulation} and a specified reward equal to $f \in [0, f_{\max}]$. We will also use the concept of the $\iota$-Gap condition which intuitively requires the existence of a unique optimal action at each state which performance is sufficiently better from the second best action. Using the $\iota$-gap condition in \cref{def:iota-gap}, we prove \cref{th:theorem-2} (see \cref{Appendix-B} for the full proof) which allows us to bound the reward and the constraint value function for data dependent policies $\bpit$.

\begin{definition}[$\iota$-gap Condition]\label{def:iota-gap}
    We define $\hmdpd_f$ as the MDP with an empirically estimated transition matrix $\hat{\cP}$ and reward function $f$. We also define $Q_f^*(s, a)$ to be the $Q$-value associated with state-action pair $(s, a)$, when following the optimal policy $\pi_f^*$ of $\hmdpd_f$. We say that $\hmdpd_f$ \textbf{satisfies the $\iota$-gap condition} if for all states $s$,
    \[
        \hV_f^*(s) - \max_{a': a \neq \hat{\pi}_f^*(s)} \hat{Q}_f^{*}(s, a') \,\geq\, \iota
    \]
    where $\hpi_f^* \coloneq \argmax_\pi \hV_f^\pi$ and $\hpi_f^*(s) = \argmax_a \hat{Q}_f^{*}(s,a)$ is the optimal action in state s.
\end{definition}

\begin{theorem} \label{th:theorem-2}
    For $\delta \in (0, 1)$, $\omega \leq 1$ and $C(\delta) =  72 \log\left( \frac{16(1 + \omega + dU) |\cS| |\cA| \log(e/1-\gamma)}{(1-\gamma)^2 \iota\delta} \right)$, where $\iota = \frac{\omega \delta (1-\gamma)\epsilon_1^d}{30 U^d |\cS| |\cA|^2}$, if $N \geq \frac{4 C(\delta/d)}{1-\gamma}$, then for $\bpit$ output by Alg. \ref{alg:alg1}, with probability at least $1 - \delta/5$,
    \begin{align*}
        \left| \Vrpbpi - \hVrpbpi \right| \,\leq\, 2\sqrt{\frac{C(\delta)}{N(1 - \gamma)^3}};\quad \left| \Vcibpi - \hVcibpi \right| \,\leq\, \sqrt{\frac{C(\delta/d)}{N(1 - \gamma)^3}} \; \fac
    \end{align*}
\end{theorem}

\textit{Proof sketch.}  Utilizing \cref{lm:union-bound-all-cmdps-satifsy-separation} we know that the empirical MDP $\hat{M}_{r_p + \blttop \cc}$ satisfies the gap condition with $\iota = \frac{\omega \delta (1 - \gamma) \epsilon_1^d}{30 U^d |\cS| |\cA|^2}$. Therefore we can apply \cref{lm:concentration-bound} with $f = r_p + \blttop \cc$ and $g = r_p$ and bound the term $\left| \Vrpbpi - \hVrpbpi \right|$ with probability $1  -\delta$. Applying \cref{lm:concentration-bound} again with $f = r_p + \blttop \cc$ and $g = c_i, \, i \in [d]$ we bound each $\left| \Vcibpi - \hVcibpi \right|$ $\fac$ with probability $1 - \frac{\delta}{10d}$ and therefore bound the summation of the data dependent policies with probability $1-\delta/5$ concluding the proof.

To also bound the non-data dependent policies $\pis$ and $\pics$ we use \cref{lm:data-dep-policies-concentration-bounds}(see \cref{Appendix-B}). The combination of \cref{th:theorem-2} and \cref{lm:data-dep-policies-concentration-bounds} allows us to bound all the terms in \cref{eq:relaxed-decomposition} and \cref{eq:strict-decomposition} completing the proof of \cref{th:theorem-3} and \cref{th:theorem-4}.

\section{Relaxed Feasibility Bounds}

For the relaxed feasibility setting, \cref{th:theorem-3} shows that for $N = \cOT \left( \frac{d |\cS| |\cA| \log(1/\delta)}{(1-\gamma)^3\epsilon^2} \right)$ samples in total, the conditions of \cref{eq:relaxed-CMDP-formulation} will be satisfied for the mixture policy $\bpit$. We prove the theorem in Appendix \ref{Appendix-C} and provide a proof sketch below.

\begin{theorem}\label{th:theorem-3}
    For a fixed $\epsilon \in \left(0, \frac{1}{1 - \gamma}\right]$ and $\delta \in (0, 1)$, Alg. \ref{alg:alg1} with $N \,=\, \cOT \left( \frac{d \log(1/\delta)}{(1-\gamma)^3\epsilon^2}  \right)$ samples, $\bb' \,=\, \bb - \frac{3\epsilon}{8}$, $\omega = \frac{\epsilon(1-\gamma)}{8}$, $U \,=\, \cO \left(\frac{1}{\epsilon(1 - \gamma)}\right)$, $\enet = \cO \left( \epsilon^2(1-\gamma)^2 \right)$ and $T = \cO \left( \frac{d^2}{(1-\gamma)^4\epsilon^4} \right)$, returns a policy $\bpit$ that satisfies the objective in Eq. \ref{eq:relaxed-CMDP-formulation} with probability at least $1 - 3\delta$.
\end{theorem}

\textit{Proof sketch.} We will show the result for a general primal-dual error $\eopt < \epsilon$ and $\bb' = \bb - \frac{\epsilon - \eopt}{2}$ and then later specify $\eopt$. To satisfy the conditions of the relaxed feasibility setting, we must show that $\bpit$ violates the constraints by at most $\epsilon$, and that $\Vrbpi$ is $\epsilon$-optimal. To show this, we invoke \cref{lm:decomposing-relaxed-suboptimality}, which shows that when certain conditions are met, then (i) $\fac:\;\, \Vcibpi \geq b_i - \epsilon$, and (ii) the reward suboptimality can be decomposed as:
\begin{equation}
    \Vrpis - \Vrbpi \,\leq\, \frac{2 \omega}{1 - \gamma} + \eopt + \left| \Vrppis - \hVrppis \right| + \left| \Vrpbpi - \hVrpbpi \right| \label{eq:relaxed-decomposition}
\end{equation}

\cref{lm:decomposing-relaxed-suboptimality} shows that $\bpit$ satisfies the relaxed constraints, and all that's left to bound the reward suboptimality. This can be achieved by upper-bounding each term of \cref{eq:relaxed-decomposition} by $\nicefrac{\epsilon}{4}$. The first term can be bounded by setting $\omega$ to $\nicefrac{\epsilon(1-\gamma)}{8}$, while for the second we set $\eopt = \nicefrac{\epsilon}{4}$. To bound the third term, since $\pis$ is a data-independent policy we can use \cref{lm:data-dep-policies-concentration-bounds} and set $N$ accordingly. To bound the fourth term we use the result of \cref{th:theorem-2} from the previous section, and set $N$ accordingly.

To instantiate \cref{alg:alg1}, we still need to compute $U$. To do this, we can use case i) of \cref{lm:dual-variable-bound}, where we show that when a feasible policy exists (e.g. $\pis$), then we can lower bound the largest Lagrange multiplier by: $\blsnorm \leq \frac{2(1+\omega)}{\epsilon'(1-\gamma)}$ for some $\epsilon'$ such that $\bb' = \bb - \epsilon'$.

Putting everything together, we require:
\begin{align}
    &\overbrace{\overbrace{\left| \Vcipis - \hVcipis \right| \,\leq\, \frac{3\epsilon}{16}}^{\text{to use case i) of \cref{lm:dual-variable-bound}}};\quad \left| \Vcibpi - \hVcibpi \right| \,\leq\, \frac{3\epsilon}{8}}^{\text{to use \cref{lm:decomposing-relaxed-suboptimality}}}; \nonumber \\[1ex]
    &\underbrace{\left| \Vrppis - \hVrppis \right| \,\leq\, \frac{\epsilon}{4};\hspace{15pt} \left| \Vrpbpi - \hVrpbpi \right| \,\leq\, \frac{\epsilon}{4}}_{\text{to bound suboptimality in \cref{eq:relaxed-decomposition}}}
\end{align}

By setting N accordingly such that each concentration bound is satisfied, we conclude that $N \geq \cOT \left( \frac{d \log(1/\delta)}{(1-\gamma)^3\epsilon^2} \right)$, and since we have $|\cS| |\cA|$ state-action pairs, we finally get $\cOT \left( \frac{d |\cS| |\cA| \log(1/\delta)}{(1-\gamma)^3\epsilon^2} \right)$ samples are required in total.

Comparing to \citep{hasanzadezonuzy2021model}, our result has an improved dependency on the cardinality of the state space ($|\cS|$ vs $|\cS|^2$), but has a worse dependency on the number of constraints ($d$ vs $\log d$). Nevertheless, for most applications it will hold that $d \ll |S|$, as designing constraint functions is usually done manually. We also compare our result to the work of \citep{bai2022achieving}, who only solve the strict feasibility setting (which also translates to a solution for the relaxed setting). Specifically, we have a better dependency on the effective horizon ($\nicefrac{1}{(1-\gamma)^3}$ vs $\nicefrac{1}{(1-\gamma)^6}$), and also we don't have a dependency on the slater constant. More details can be found on \cref{tab:results-comparison}.

\section{Strict Feasibility Bounds}

In the strict feasibility setting, no constraint violations are allowed. Thus, we have to ensure that $\Vcibpi \geq b_i,\, \fac$. To do this, we'll follow a similar approach to \cref{th:theorem-3}, but instead of setting $\bb' < \bb$, we will set $\bb' > \bb$. Since we don't know for which $\bb' > \bb$ the problem is feasible, we have to somehow compute an equivalent of the feasible region (i.e. the slater constant).

Let $\tpi$ be a feasible policy for the original CMDP, i.e. $\Vcitpi \geq b_i,\; \fac$. We define the margin as $\slaterctpi \coloneq \min_{i \in [d]} \left\{ \Vcitpi - b_i \right\}$. Let $\pics$ be the policy that achieves the largest margin, i.e. $\pics \coloneq \argmax_\pi \left\{ \min_{i \in d} \left\{ \Vcipi - b_i \right\} \;|\; \Vcipi \geq b_i \right\}$, and $\slatercs \coloneq \slaterc^{\pics} $ be the corresponding slater constant. For our CMDP we will assume \textbf{strict feasibility}, i.e. $\slatercs > 0$.

With these definitions in mind, we present \cref{th:theorem-4} which shows that for $N = \cOT \left(\frac{d^3 \log(1/\delta)}{(1-\gamma)^5 \epsilon^2 {\slatercs}^2} \right)$ samples in total, the conditions of \cref{eq:strict-CMDP-formulation} will be satisfied for the mixture policy $\bpit$. We prove the theorem in Appendix \ref{Appendix-D} and provide a proof sketch below. Note that in practice it's hard to estimate $\slatercs$, but still any lower bound $\slaterctpi$ such that $0 < \slaterctpi \leq \slatercs$ will work (with adapted complexities).

\begin{theorem}\label{th:theorem-4}
    For a fixed $\epsilon \in \left(0, \frac{1}{1 - \gamma}\right]$ and $\delta \in (0, 1)$, Alg. \ref{alg:alg1} with $N \,=\, \cOT \left(\frac{d^3 \log(1/\delta)}{(1-\gamma)^5 \epsilon^2 {\slatercs}^2} \right)$ samples, $\bb' \,=\, \bb + \frac{\epsilon(1-\gamma)\slatercs}{20}$, $\omega = \frac{\epsilon(1-\gamma)}{10}$, $U \,=\, \frac{4(1+\omega)}{\slatercs(1-\gamma)}$, $\enet = \cO \left( \epsilon^2 (1-\gamma)^4 {\slatercs}^2 \right)$ and $T = \cO \left(\frac{d^4}{(1-\gamma)^6 \epsilon^2 {\slatercs}^2 }\right)$, returns a policy $\bpit$ that satisfies the objective in Eq. \ref{eq:strict-CMDP-formulation} with probability at least $1 - 4\delta$.
\end{theorem}

\textit{Proof sketch.} We will show the result for a general $\bb' = \bb - \Delta$ for $\eopt < \Delta$ and then later specify $\Delta$ and $\eopt$. To satisfy the conditions of the strict feasibility setting, we must show that $\bpit$ does not violate any constraint, and that $\Vrbpi$ is $\epsilon$-optimal. To show this, we invoke \cref{lm:decomposing-suboptimality-strict}, which shows that when certain conditions are met, then (i) $\fac:\;\, \Vcibpi \geq b_i$, and (ii) the reward suboptimality can be decomposed as:
\begin{equation}
    \Vrpis - \Vrbpi \;\leq\; \frac{2\omega}{1-\gamma} + \eopt + 2d\lambda_{\min}^*\Delta + \left[ \Vrppis - \hVrppis \right] + \left[ \Vrpbpi - \Vrpbpi \right] \label{eq:strict-decomposition}
\end{equation}

\cref{lm:decomposing-suboptimality-strict} shows that $\bpit$ satisfies the relaxed constraints, and all that's left to bound the reward suboptimality. This can be achieved by upper-bounding each term of \cref{eq:strict-decomposition} by $\nicefrac{\epsilon}{5}$. The first term can be bounded by setting $\omega$ to $\nicefrac{\epsilon(1-\gamma)}{10}$, while for the second we set $\eopt = \nicefrac{\Delta}{5} \leq \nicefrac{\epsilon}{5}$. The fourth and fifth term can be again bounded using \cref{lm:data-dep-policies-concentration-bounds} and \cref{th:theorem-2} by setting $N$ accordingly.

To instantiate \cref{alg:alg1}, we still need to compute $U$. To do this, we can use case ii) of \cref{lm:dual-variable-bound}, where we show that when a feasible policy exists (e.g. $\pics$), then we can lower bound the infinity norm of the vector containing the dual variables, $\bl$, by: $\blsnorm \leq \frac{2(1+\omega)}{\slatercs(1-\gamma)}$.

Putting everything together, we require:
\begin{align}
    &\overbrace{\left| \Vcipis - \hVcipis \right| \,\leq\, \Delta;\quad \left| \Vcibpi - \hVcibpi \right| \,\leq\, \frac{4\Delta}{5}}^{\text{to use \cref{lm:decomposing-suboptimality-strict}}};\quad \overbrace{\left| \Vcipics - \hVcipics \right| \,\leq\, \frac{19 \Delta}{5};}^{\text{to use case ii) of \cref{lm:dual-variable-bound}}} \nonumber \\[1ex]
    &\underbrace{\left| \Vrppis - \hVrppis \right| \,\leq\, \frac{\epsilon}{5};\hspace{11pt} \left| \Vrpbpi - \hVrpbpi \right| \,\leq\, \frac{\epsilon}{5}}_{\text{to bound suboptimality in \cref{eq:strict-decomposition}}}
\end{align}

By setting N accordingly such that each concentration bound is satisfied, we conclude that $N \geq \cOT \left( \frac{d^3 \log(1/\delta)}{(1-\gamma)^5\epsilon^2{\slatercs}^2} \right)$, and since we have $|\cS| |\cA|$ state-action pairs, we finally get $\cOT \left( \frac{d^3 |\cS| |\cA| \log(1/\delta)}{(1-\gamma)^5\epsilon^2{\slatercs}^2} \right)$ samples are required in total.

Comparing to \citep{bai2022achieving}, we can see that we have a better dependency on the effective horizon ($\nicefrac{1}{(1-\gamma)^5}$ vs $\nicefrac{1}{(1-\gamma)^6}$), but we have a worse dependency on the number of constraints ($d^3$ vs $d$). More details can be found on \cref{tab:results-comparison}.

\renewcommand{\arraystretch}{1.6}
\begin{table}[t]
    \hspace*{-12mm}
    \begin{tabular}{|c|c|c|c|c|}
        \hline
        \textbf{Algorithm} & \textbf{Constraints} & \textbf{Relaxed Feasibility} & \textbf{Strict Feasibility}  & \textbf{$\epsilon$-range} \\
        \hline
        \citep{vaswani2022near} & 1 & $\cOT\left( \frac{|\cS| |\cA| \log\left(\nicefrac{1}{\delta}\right)}{(1-\gamma)^3\epsilon^2} \right)$ & $\cOT \left( \frac{|\cS| |\cA| \log\left( \nicefrac{1}{\delta} \right)}{(1-\gamma)^5 \epsilon^2 \zeta^2} \right)$ & $\left( 0, \frac{1}{1 - \gamma} \right]$ \\[1ex]
        \hline
        \citep{hasanzadezonuzy2021model} & $d > 1$ & $\cOT\left( \frac{\gamma^2 |\cS|^2 |\cA| \log\left(\nicefrac{d}{\delta}\right)}{(1-\gamma)^3\epsilon^2} \right)$ & - &  $\left( 0, \frac{0.22 \gamma}{\sqrt{|\cS|(1-\gamma)}} \right)$ \\[1ex]
        \hline
        \citep{bai2022achieving} & $d > 1$ & $\cOT\left( \frac{d |\cS| |\cA|}{(1-\gamma)^6 \epsilon^2 \zeta^2} \right)$ & $\cOT\left( \frac{d |\cS| |\cA|}{(1-\gamma)^6 \epsilon^2 \zeta^2} \right)$ &  $\left(0, \frac{1}{1-\gamma} \right]$ \\[1ex]
        \hline
        \textbf{Ours} & $d > 1$ & $\cOT\left( \frac{d |\cS| |\cA| \log\left(\nicefrac{1}{\delta}\right)}{(1-\gamma)^3\epsilon^2} \right)$ & $\cOT\left( \frac{d^3 |\cS| |\cA| \log\left(\nicefrac{1}{\delta}\right)}{(1-\gamma)^5 \epsilon^2 {\slatercs}^2} \right)$ &  $\left( 0, \frac{1}{1 - \gamma} \right]$ \\
        \hline
    \end{tabular}
    \vspace{4mm}
    \caption{Comparison with prior results that use a generative model for both feasibility settings.}
    \label{tab:results-comparison}
\end{table}

\section{Discussion}
By expanding the work of \cite{vaswani2022near}, we were to able to achieve a natural $\cOT \left( \frac{d |\cS| |\cA| \log(1/\delta)}{(1-\gamma)^3\epsilon^2} \right)$ sample complexity bound for the relaxed feasibility setting. On the other hand, the bound for the strict feasibility setting $\cOT \left( \frac{d^3 |\cS| |\cA| \log(1/\delta)}{(1-\gamma)^5\epsilon^2{\slatercs}^2} \right)$ is bottlenecked by the cubic term $d^3$ (note that both bounds recover the original sample complexities for $d = 1$). The work of \cite{bai2022achieving}, suggests that this dependence might be sub-optimal. The additional square dependence (compared to the relaxed setting) arises from the fact that the union-bound has to cover all possible relaxed MDPs, and due to the nature of our discretization there's an exponential number of them. While this term is inside a logarithm (see \cref{th:theorem-2}), it later gets squared in \cref{Appendix-D}. This introduces a multiplicative $d^2$ term, which together with the multiplicative $d$ term from the relaxed-feasibility analysis, yields the cubic term.

\section{Conclusion}
In this work we provided sample complexity bounds for CMDP with multiple constraints, generalizing the results of \cite{vaswani2022near}. We adopted the model-based primal-dual algorithm proposed in \cite{paternain2019constrained} which solves the CMDP by iteratively solving an unconstrained MDP problem and updating its primal and dual variables and analyzed both the relaxed and the strict settings. For the former we obtained a $\cOT \left( \frac{d |\cS| |\cA| \log(1/\delta)}{(1-\gamma)^3\epsilon^2} \right)$ sample complexity. For the latter much more difficult case, which does not allow any constraint violation we obtained a $\cOT \left( \frac{d^3 |\cS| |\cA| \log(1/\delta)}{(1-\gamma)^5\epsilon^2{\slatercs}^2} \right)$ sample complexity under the assumption that the term $\slatercs$ can be efficiently estimated. In the future, we aim to quantify the difficulty of estimating the term $\slatercs$ and to try to obtain better bounds for the relaxed feasibility case that have logarithmic dependence in the number of constraints, as in \citep{hasanzadezonuzy2021model}.

\section{Acknowledgements}

We would like to thank Pragnya Alatur for her constructive feedback during the development of this work, as well as her significant contributions to it.

\clearpage
\bibliographystyle{plainnat}
\bibliography{ref.bib}

\clearpage
\appendix

\setcounter{theorem}{0}

\input{appendices/primal-dual-appendix}
\input{appendices/concentration-appendix}
\input{appendices/relaxed-feasibility-appendix}
\input{appendices/strict_feasibility-appendix}

\end{document}

%% file: shortcuts.tex
\newcommand{\E}{\mathop{\mathbb{E}}}  
\newcommand{\R}{\mathbb{R}}           
\newcommand{\ProjP}{\mathbb{P}}       
\newcommand{\mdp}{M}                  
\newcommand{\mdpd}{\mdp^d}            
\newcommand{\hmdp}{\hat{M}}           
\newcommand{\hmdpd}{\hmdp^d}          

\newcommand{\cS}{\mathcal{S}}           
\newcommand{\cA}{\mathcal{A}}           
\newcommand{\cP}{\mathcal{P}}           
\newcommand{\cU}{\mathcal{U}}           
\newcommand{\cL}{\mathcal{L}}           
\newcommand{\cE}{\mathcal{E}}           
\newcommand{\cR}{\mathcal{R}}           
\newcommand{\cO}{\mathcal{O}}           
\newcommand{\cOT}{\tilde{\mathcal{O}}}  

\newcommand{\bl}{\boldsymbol{\lambda}}               
\newcommand{\blt}{\bl_t}                             
\newcommand{\bltt}{\bl_{t+1}}                        
\newcommand{\bltop}{\bl^\top}                        
\newcommand{\blttop}{\bl_t^\top}                     
\newcommand{\bls}{\bl^*}                             
\newcommand{\blsnorm}{\left\| \bls \right\|_\infty}  
\newcommand{\blstop}{{\bls}^\top}                    
\newcommand{\bt}{\boldsymbol{\tau}}                  
\newcommand{\tilbt}{\Tilde{\bt}}                     
\newcommand{\bb}{\mathbf{b}}                         
\newcommand{\cc}{\mathbf{c}}                         
\newcommand{\xx}{\mathbf{x}}                         
\newcommand{\li}{\lambda_i}                          
\newcommand{\lis}{\li^*}                             
\newcommand{\lti}{\lambda_{t,i}}                     
\newcommand{\ltti}{\lambda_{t+1,i}}                  
\newcommand{\enet}{\epsilon_1}                       
\newcommand{\eopt}{\epsilon_{\mathrm{opt}}}          
\newcommand{\slaterc}{\zeta_{\cc}}                   
\newcommand{\slaterctpi}{\slaterc^{\tpi}}            
\newcommand{\slatercs}{\slaterc^*}                   

\newcommand{\fac}{\forall i \in [d]}  

\newcommand{\pis}{\pi^*}                
\newcommand{\hpi}{\hat{\pi}}            
\newcommand{\hpit}{\hat{\pi}_t}         
\newcommand{\pics}{\pi_{\mathbf{c}}^*}  
\newcommand{\hpis}{\hat{\pi}^*}         
\newcommand{\tpi}{\Tilde{\pi}}          
\newcommand{\tpis}{\Tilde{\pi}^*}       
\newcommand{\bpi}{\bar{\pi}}            
\newcommand{\bpit}{\bar{\pi}_T}         

\newcommand{\Vc}{\mathbf{V}_{\cc}}           
\newcommand{\Vcbpi}{\Vc^{\bpit}}             
\newcommand{\Vci}{V_{c_i}}                   
\newcommand{\Vcipi}{\Vci^\pi(\rho)}          
\newcommand{\Vcipis}{\Vci^{\pis}(\rho)}      
\newcommand{\Vcipics}{\Vci^{\pics}(\rho)}    
\newcommand{\Vcibpi}{\Vci^{\bpit}(\rho)}     
\newcommand{\Vcihpi}{\Vci^{\hpi}(\rho)}      
\newcommand{\Vcitpi}{\Vci^{\tpi}(\rho)}      
\newcommand{\Vrpi}{V_r^\pi(\rho)}            
\newcommand{\Vrpis}{V_r^{\pis}(\rho)}        
\newcommand{\Vrhpi}{V_r^{\hpi}(\rho)}        
\newcommand{\Vrbpi}{V_r^{\bpit}(\rho)}       
\newcommand{\Vrp}{V_{r_p}}                   
\newcommand{\Vrppi}{\Vrp^\pi(\rho)}          
\newcommand{\Vrppis}{\Vrp^{\pis}(\rho)}      
\newcommand{\Vrptpi}{\Vrp^{\tpi}(\rho)}      
\newcommand{\Vrphpit}{\Vrp^{\hpit}(\rho)}    
\newcommand{\Vrpbpi}{\Vrp^{\bpit}(\rho)}     

\newcommand{\hV}{\hat{V}}                    
\newcommand{\hVc}{\hat{\mathbf{V}}_{\cc}}    
\newcommand{\hVcpi}{\hVc^\pi(\rho)}          
\newcommand{\hVchpit}{\hVc^{\hpit}(\rho)}    
\newcommand{\hVcbpi}{\hVc^{\bpit}(\rho)}     
\newcommand{\hVctpi}{\hVc^{\tpi}(\rho)}      
\newcommand{\hVctpis}{\hVc^{\tpis}(\rho)}    
\newcommand{\hVci}{\hat{V}_{c_i}}            
\newcommand{\hVcipi}{\hVci^\pi(\rho)}        
\newcommand{\hVcipis}{\hVci^{\pis}(\rho)}    
\newcommand{\hVcihpis}{\hVci^{\hpis}(\rho)}  
\newcommand{\hVcipics}{\hVci^{\pics}(\rho)}  
\newcommand{\hVcihpit}{\hVci^{\hpit}(\rho)}  
\newcommand{\hVcibpi}{\hVci^{\bpit}(\rho)}   
\newcommand{\hVcitpi}{\hVci^{\tpi}(\rho)}    
\newcommand{\hVcitpis}{\hVci^{\tpis}(\rho)}  
\newcommand{\hVr}{\hV_r}                     
\newcommand{\hVrpi}{\hVr^\pi(\rho)}          
\newcommand{\hVrp}{\hat{V}_{r_p}}            
\newcommand{\hVrppi}{\hVrp^\pi(\rho)}        
\newcommand{\hVrppis}{\hVrp^{\pis}(\rho)}    
\newcommand{\hVs}{\hVrp^{\hpis}(\rho)}       
\newcommand{\hVrptpi}{\hVrp^{\tpi}(\rho)}    
\newcommand{\hVrptpis}{\hVrp^{\tpis}(\rho)}  
\newcommand{\hVrphpit}{\hVrp^{\hpit}(\rho)}  
\newcommand{\hVrpbpi}{\hVrp^{\bpit}(\rho)}   

\DeclareMathOperator*{\argmax}{arg\,max}

%% file: appendices/primal-dual-appendix.tex
\section{Proofs for Primal-Dual algorithm (proof of \cref{th:theorem-1})}\label{Appendix-A}

\begin{theorem}[Guarantees for the primal-dual algorithm]
    For a target error $\eopt > 0$ and the primal-dual updates in Eq. \ref{eq:primal-update} - Eq. \ref{eq:dual-update} with $U > \blsnorm$, $T = \frac{4 U^2 d^2}{\eopt^2 (1-\gamma)^2} \left[ 1 + \frac{1}{(U - \blsnorm)^2} \right]$, $\eta =\frac{U (1-\gamma)}{\sqrt{T}}$, and $\enet = \frac{\eopt^2 (1-\gamma)^2 (U - \blsnorm)}{6dU}$, the mixture policy $\bpit \coloneq \frac{1}{T}\sum_{t=0}^{T-1} \hpit$ satisfies
    \[
        \hVrpbpi \geq \hVs - \eopt \quad \text{and} \quad \hVcibpi \geq b_i' - \eopt \; \fac
    \]
    
    where $\hVrpbpi \,=\, \frac{1}{T} \sum_{t = 0}^{T - 1} \hVrphpit$ is the value function of the mixture policy $\bpit$.
\end{theorem}

\begin{proof}
 We begin by defining the dual regret with respect to any vector of Lagrange multipliers $\bl \in \cU$ as:
\begin{equation}
    R^d(\bl, T) = \sum_{t=0}^{T-1} (\blt - \bl)^\top (\hVchpit - \bb') = \sum_{i=1}^{d}\sum_{t=0}^{T-1}(\lti - \li)(\hVcihpit - b_i') \label{eq:dual-regret}
\end{equation}

To bound Eq. \ref{eq:dual-regret} we begin by fixing an arbitrary $\li \in [0, U]$ and tracking the change of $\left| \lti - \li \right|$. Defining $\lambda_{t+1,i}' \coloneq \ProjP_{[0, U]} \left[ \lti - \eta(\hVcihpit - b_i') \right]$
\begin{align*}
    |\ltti - \li| \;
        =&\; \left| \cR_\Lambda [\ltti'] - \li \right| \;=\; \left| \cR_\Lambda [\ltti'] - \ltti' + \ltti' - \li \right| \\[1ex]
        \leq&\; \left| \cR_\Lambda [\ltti'] - \ltti' \right| + \left |\ltti' - \li \right| \\[1ex]
        \leq&\; \enet + \left| \ltti' - \li \right|
\end{align*}

where in the last inequality we used that fact that $\left| \li - \cR_\Lambda[\li] \right| \,\leq\, \enet$ for all $\li \in [0, U]$ due to the epsilon net. Squaring both sides we get,
\begin{align*}
    \left| \ltti - \li \right|^2 \;
        \leq&\; \enet^2 + \left| \ltti' - \li \right|^2 + 2\enet \left| \ltti' - \li \right| \\[1ex]
        \leq&\; \enet^2 + 2\enet U + \left| \ltti' - \li \right|^2 \tag{since $\li, \ltti' \in [0,U])$} \\[1ex]
        \leq&\; \enet^2 + 2\enet U + \left| \lti - \eta \left( \hVcihpit - b_i' \right) - \li \right|^2 \tag{since projections are non expansive} \\[1ex]
        =&\; \enet^2 + 2\enet U + \left| \lti - \li \right|^2 - 2\eta \left( \lti - \li\right) \left( \hVcihpit - b_i' \right) + \eta^2 (\hVcihpit - b_i')^2 \\[1ex]
        \leq&\; \enet^2 + 2\enet U + \left| \lti - \li \right|^2 - 2\eta \left( \lti - \li\right) \left( \hVcihpit - b_i'\right) + \frac{\eta^2}{(1 - \gamma)^2}
\end{align*}

where in the last inequality we used the fact that $0 \leq \hVcihpit, b_i' \leq \frac{1}{1-\gamma}$. Rearranging and dividing by $2 \eta$ we get
\[
    \left( \lti - \li \right) \left( \hVcihpit - b_i' \right) \;\leq\; \frac{\enet^2 + 2\enet U}{2\eta} + \frac{|\lti - \li|^2 - |\ltti - \li|^2}{2\eta} + \frac{\eta}{2(1-\gamma)^2}
\]

Summing from $t = 0$ to $T - 1$:
\begin{align*}
    \sum_{t=0}^{T-1} \left( \lti - \li \right) \left( \hVcihpit - b_i' \right) \;
        \leq&\; T\frac{\enet^2 + 2\enet U}{2\eta} + \frac{1}{2\eta} \sum_{t = 0}^{T - 1} \bigg[ |\lti - \li|^2 - |\ltti - \li|^2 \bigg] + \frac{\eta T}{2(1-\gamma)^2} \\[1ex]
        \leq&\; T\frac{\enet^2 + 2\enet U}{2\eta} + \frac{U^2}{2\eta} + \frac{\eta T}{2(1-\gamma)^2}
\end{align*}

where for the last inequality we used telescoping summation, bounded the $|\lambda_{0,i} - \li|$ term by $U$ and dropped the negative term $-|\lambda_{T, i} - \li|$.

Setting $\eta = \frac{U (1-\gamma)}{\sqrt{T}}$:
\begin{equation*}
    \sum_{t=0}^{T-1}(\lti - \li)(\hVcihpit - b_i') \;\leq\; T^{3/2}\frac{\enet^2 + 2\enet  U}{2U(1-\gamma)} + \frac{U \sqrt{T}}{1-\gamma}
\end{equation*}

Summing over all Lagrange multipliers $i = 1$ to $d$:
\[
    \sum_{i=1}^{d} \sum_{t=0}^{T-1} (\lti - \li) (\hVcihpit - b_i') \;\leq\; d \left[ T^{3/2}\frac{\enet^2 + 2\enet U}{2U(1-\gamma)} + \frac{U \sqrt{T}}{1-\gamma} \right]
\]

Therefore we get the following upper-bound for the dual regret:
\begin{equation}
    R^d(\bl, T) \;\leq\; d \left[ T^{3/2}\frac{\enet^2 + 2\enet U}{2U(1-\gamma)} + \frac{U \sqrt{T}}{1-\gamma} \right] \label{eq:dual-regret_bound}
\end{equation}

Having upper-bounded the dual regret, we will now upper bound the Lagrangian with the dual regret. Utilizing the primal update in Eq. \ref{eq:primal-update} for any policy $\pi$:
\begin{equation*}
    \hVrphpit + \bl\top \hVchpit \;\geq\; \hVrp^{\pi}(\rho) + \bltop \hVcpi
\end{equation*}

Substituting $\pi = \hpis$ and recalling that $\hpis$ is a solution for the empirical CMDP, i.e. it holds that: $\hVcihpis \geq b_i',\; \fac$, we have:
\begin{equation}
    \hVrppis - \hVrphpit \;\leq\; \bl_t [\hVchpit - \bb'] \label{eq:pi-star-bound}
\end{equation}

Starting from the definition of the dual regret Eq. \ref{eq:dual-regret} and utilizing the Eq. \ref{eq:pi-star-bound} and dividing by $T$ gives 
\begin{equation}
    \frac{1}{T} \sum_{t=0}^{T-1} [\hVs - \hVrphpit] + \frac{1}{T} \sum_{t=0}^{T-1} \bltop (\bb' - \hVchpit) \;\leq\; \frac{R^d(\bl, T) }{T} \label{eq:mixture-bound}
\end{equation}

We recall that $\bpit = \frac{1}{T}\sum_{t=0}^{T-1} \hpit$. Thus, by definition of this mixture policy:
\[
    \frac{1}{T}\sum_{t=0}^{T-1} \hVrphpit \,=\, \hVrpbpi, \quad\;\,
    \frac{1}{T}\sum_{t=0}^{T-1} \hVchpit \,=\, \hVcbpi
\]

Thus, we can rewrite \ref{eq:mixture-bound} as:
\begin{equation}
    \bigg[ \hVs - \hVrpbpi \bigg] + \bltop (\bb' - \hVcbpi) \;\leq\; \frac{R^d(\bl, T)}{T} \label{eq:bound-1}
\end{equation}
which holds for all $\bl \in \cU$.

The bound of the dual regret Eq. \ref{eq:dual-regret_bound} combined with the previous inequality and the right choice of $T$, number of updates and $\enet$ will give us the desired bounds for the $\hVs - \hVrpbpi$ and $\bb' - \hVcbpi$ terms. Specifically for the reward optimally gap, for $\bl = 0_d \in [0, U]^d$ we have

\begin{align}
    \hVs - \hVrpbpi \;
        \leq&\; d \left[ \sqrt{T} \frac{\enet^2 + 2\enet U}{2U (1-\gamma)} + \frac{U}{(1-\gamma)\sqrt{T}} \right] \\[1ex]
        <&\; \sqrt{T}\frac{3\enet d}{2U (1-\gamma)} + \frac{dU}{(1-\gamma)\sqrt{T}}\label{eq:theorem1-bound1}
\end{align}

For the constraint violation there are two cases. In the case where no constraints are violated, i.e. $b_i' - \hVcibpi \leq 0, \; \fac$ we observe that for all constraints it will hold that $b_i' - \eopt - \hVcibpi \leq 0, \fac$, which is what we wanted to prove.

The second case is when we have at least one constraint violation, i.e. at least one positive term $b_i' - \hVcibpi > 0$ for some $i \in [d]$. In this case we will proceed to upper bound all individual terms. Combining \ref{eq:dual-regret_bound} and \ref{eq:bound-1} we get $\forall \bl \in \cU$:
\begin{equation}
    \bigg[ \hVs - \hVrpbpi \bigg] + \bltop (\bb' - \hVcbpi) \;\leq\; d \left[ T^{1/2}\frac{\enet^2 + 2\enet U}{2U(1-\gamma)} + \frac{U}{(1-\gamma)\sqrt{T}} \right] \label{ineq:bound-used-for-lemma1}
\end{equation}

Setting $\beta$ to be the right-hand side of \ref{ineq:bound-used-for-lemma1}, we apply \cref{lm:lemma-violation-threshold} to get $\fac$:
\begin{equation}
    (b_i' - \hVcibpi \;\leq\; d\sqrt{T}\frac{\enet^2 + 2\enet U}{2U(1-\gamma)(U - \blsnorm)} + \frac{dU}{(1-\gamma)(U - \blsnorm)\sqrt{T}}\label{ineq:bound-individual-constraints-1}
\end{equation}

Using the fact that $\enet < U$, we get from \ref{ineq:bound-individual-constraints-1}:
\begin{equation}
    (b_i' - \hVcibpi \;<\; \sqrt{T}\frac{3\enet d}{2(1-\gamma)(U - \blsnorm)} + \frac{dU}{(1-\gamma)(U - \blsnorm)\sqrt{T}} \label{eq:theorem1-bound2}
\end{equation}

We continue by setting $T$ such that the second terms in Eq. \ref{eq:theorem1-bound1}, Eq. \ref{eq:theorem1-bound2} are bounded by $\frac{\eopt}{2}$ which gives us
\[
    T \,=\, T_0 \,\coloneq\, \frac{4U^2d^2}{\eopt^2 (1-\gamma)^2} \left[ 1 + \frac{1}{(U - \blsnorm)^2}\right]
\]

Using this $T_0$, along with the fact that $\sqrt{a^2 + b^2} \leq a + b, \;\forall a, b \geq 0$, we can further simplify Eq. \ref{eq:theorem1-bound1} and Eq. \ref{eq:theorem1-bound2} $\fac$:
\begin{align*}
    \hVs - \hVrpbpi \;\leq&\; \frac{2dU}{(1-\gamma) \eopt} \left( 1 + \frac{1}{U - \blsnorm} \right) \frac{3\enet}{2 (1-\gamma)} + \frac{\eopt}{2} \\[1ex]
    b_i' - \hVcibpi \;\leq&\; \frac{2dU}{(1-\gamma) \eopt} \left( 1 + \frac{1}{U - \blsnorm} \right) \frac{3\enet}{2 (1-\gamma)(U - \blsnorm)} + \frac{\eopt}{2}
\end{align*}

Finally by setting $\enet$ such that the first term in the above inequalities are bounded from above by $\frac{\eopt}{2}$ we get that:
\[
    \enet \;=\; \frac{\eopt^2 (1-\gamma)^2(U - \blsnorm)}{6dU}
\]

The above value finishes the proof as we can now ensure that:
\[
    \hVs - \hVrpbpi \;\leq\; \eopt \quad \text{and} \quad b_i' - \hVcibpi \;\leq\; \eopt \; \fac
\]
\end{proof}

\vspace{10mm}

\begin{lemma}[Adapted from Lemma 10 \citep{vaswani2022near}]\label{lm:lemma-violation-threshold}
    If for some policy $\tpi$ and any $\bl \in \cU = [0, U]^d$ where $U > \blsnorm$, it holds that:
    \[
        \hVs - \hVrptpi + \bltop \left( \bb ' - \hVctpi \right) \leq \beta
    \]
    for some $\beta \in \R$, then we have that:
    \[
        b_i' - \hV_{c_i}^{\Tilde{\pi}}(\rho) \leq \frac{\beta}{U - \| \bl^* \|_\infty},\; \forall i \in [d]
    \]
\end{lemma}

\begin{proof}
    We begin by defining for $\bt \coloneq [\tau_1, \dots, \tau_d]^\top$:
\[
    \nu(\bt) \coloneq \max_\pi \left\{ \hVrppi \;|\; \hVcipi \geq b'_i + \tau_i,\; \fac \right\}
\]
and note that by definition, $\nu (0_d) = \hVs$. Furthermore, we can observe that $\nu$ is decreasing with respect to each of its coordinates (i.e. if we increase/decrease one coordinate while keeping the others fixed, then $\nu$ decreases/increases).

\bigskip
Let $\hV_l^{\pi, \bl}(\rho) \coloneq \hVrppi + \bltop \left( \hVcpi - \bb' \right)$. Then, for any $\bt$ and policy $\pi$ s.t.
\begin{equation}
    \hVcipi \geq b_i' + \tau_i,\; \fac \label{eq:lemma-viol-thresh-assumption}
\end{equation}
we have
\begin{align}
    \hV_l^{\pi, \bls}(\rho) \;\leq&\; \max_{\pi'} \hV_l^{\pi', \bls}(\rho) \nonumber \\[1ex]
        =&\; \hVs \tag{by strong duality} \nonumber \\[1ex]
        =&\; \nu(0_d) \nonumber \\[1.5ex]
    \Longrightarrow\;\, \nu(0_d) - \blstop \bt \,\geq&\, \hV_l^{\pi, \bls}(\rho) \,-\, \bltop \bt \nonumber \\[1ex]
        =&\; \hVrppi (\rho) + \blstop \left( \hVcpi - \bb' - \bt \right) \nonumber \\[1ex]
        \geq&\; \max_\pi \{ \hVrppi \;\,|\,\; \hVcipi \geq b'_i + \tau_i,\; \fac \} \nonumber \tag{by assumption \ref{eq:lemma-viol-thresh-assumption} + inequality holds for $\argmax_\pi$} \\[1ex]
        =&\; \nu(\bt) \nonumber \\[1ex]
    \Longrightarrow\;\, \blstop \bt \,\leq\, \nu(&0_d) \,-\, \nu(\bt) \label{eq:lemma-viol-thresh-bound}
\end{align}

Now we choose $\tilbt = \bb' - \hVctpi$ and we compute:
\begin{align}
    (\bl - \bls)^\top \tilbt \;=&\; -\blstop \tilbt + \bltop \tilbt \nonumber \\[1ex]
        =&\; \blstop (-\tilbt) + \bltop \tilbt \nonumber \\[1ex]
        \leq&\; \nu(0_d) - \nu(-\tilbt) + \bltop \tilbt \tag{Eq. \ref{eq:lemma-viol-thresh-bound}} \nonumber \\[1ex]
        =&\; \hVs - \nu(-\tilbt) + \bltop \tilbt \tag{$\nu(0_d)$} \nonumber \\[1ex]
        =&\; \hVs - \nu(-\tilbt) + \bltop \tilbt + \hVrptpi - \hVrptpi \tag{$\pm\; \hVrptpi$} \nonumber \\[1ex]
        =&\; \hVs - \hVrptpi + \bltop \left( \bb' - \hVctpi \right) + \hVrptpi - \nu(-\tilbt) \nonumber \\[1ex]
        \leq&\; \beta + \hVrptpi - \nu(-\tilbt) \label{eq:lemma-viol-thresh-diff-bound}
\end{align}

Where for the last inequality we used the lemma assumption. But now we can bound $\nu(-\tilbt)$ as follows:
\begin{align}
    \nu(-\tilbt) \;=&\; \max_\pi \left\{ \hVrppi \,|\, \hVcipi \,\geq\, b_i' - \Tilde{\tau}_i,\; \fac \right\} \nonumber \\[1ex]
        =&\; \max_\pi \left\{ \hVrppi \,|\, \hVcipi \,\geq\, b_i' - (b_i' - \hVcitpi),\; \fac \right\} \nonumber \\[1ex]
        =&\; \max_\pi \left\{ \hVrppi \,|\, \hVcipi \,\geq\, \hVcitpi,\; \fac \right\} \label{eq:lemma-viol-thresh-nu-max-bound}
\end{align}

It's obvious to see that the condition of the maximum of \ref{eq:lemma-viol-thresh-nu-max-bound} will hold for $\pi = \tpi$, and thus the whole maximum will be greater or equal than the argument evaluated at $\pi = \tpi$:
\begin{equation}
    \nu(-\tilbt) \;\geq\; \hVrptpi \label{eq:lemma-viol-thresh-nu-bound}
\end{equation}

We recall the definition of $\tilbt = \bb' - \hVctpi$. Combining this with \ref{eq:lemma-viol-thresh-diff-bound} and \ref{eq:lemma-viol-thresh-nu-bound} we get:
\begin{equation}
    \left( \bl - \bls \right)^\top \left( \bb' - \hVctpi \right) \;\leq\; \beta,\quad \forall \lambda \in \cU \label{eq:lemma-viol-thresh-l-bound}
\end{equation}

To finally show the lemma, we consider inequality \ref{eq:lemma-viol-thresh-l-bound} for all indices $d$ times, once for each index $i = 1, \dots, d$ with the corresponding $\bl^{* - i}$:
\[
    \bl^{* - i} \,\coloneq\, [\lambda_1^*, \dots, \overbrace{U}^{\text{index } i}, \dots, \lambda_d^*] = \bls + (U - \lambda_i^*) \mathbf{e}_i
\]

Plugging in $\bl = \bl^{* - i}$ in \ref{eq:lemma-viol-thresh-l-bound} cancels out all terms expect index $i$ in the dot product, and since $U > \blsnorm$ we get $\fac$:
\begin{equation*}
    \left( U - \lambda_i^* \right) \left( b_i' - \hVcitpi \right) \,\leq\, \beta \;\,\Rightarrow\,\; b_i' - \hVcitpi \,\leq\, \frac{\beta}{U - \lambda_i^*} \,\leq\, \frac{\beta}{U - \blsnorm}
\end{equation*} \qedhere
\end{proof}

%% file: appendices/concentration-appendix.tex
\section{Concentration Proofs (proof of \cref{th:theorem-2})}\label{Appendix-B}

To prove the concentration bounds for the reward and the constraint value functions, we utilize a general unconstrained MDP  $\mdpd_f = (\cS, \cA, \cP, \gamma, f)$. The main difference between this MDP and the CMDP in \cref{eq:CMDP-formulation} is that is has a specified reward equal to $f \in [0, f_{\max}]$ independent of the sampling of the transition matrix. Apart from that we keep the same state-action space, transition probabilities and discount factor. We also define the empirical MDP $\hmdpd_f = (\cS, \cA, \hat{\cP}, \gamma, f)$ that has the same transition dynamics $\hat{\cP}$ with the empirical CMDP in \cref{eq:emprical-CMDP-formulation}. In a similar manner we can also define MDP  $\mdpd_g = (\cS, \cA, \cP, \gamma, g)$ and the empirical MDP $\hmdpd_g = (\cS, \cA, \hat{\cP}, \gamma, g)$ with rewards $g \in [0, g_{\max}]$. We further define the value functions  $V_f^\pi$ and $\hV_f^\pi$ (and $V_g^\pi$ and $\hV_g^\pi$) associated with the the policy $\pi$ in $\mdpd_f$ and $\hmdpd_f$ (as well as in $\mdpd_g$ and $\hmdpd_{g}$)  as well as the optimal value functions as $V_f^*$ and $\hV_f^{*}$ (and $V_g^{*}$ and $\hV_g^{*}$). According to this logic we also define the action-value functions for policy $\pi$ and a state-action pair $(s, a)$ in $\mdpd_f$ as $Q_f^\pi(s, a)$ and for $\hmdpd_f$ as $\hat{Q}_f^\pi(s, a)$. Lastly, we also require the MDP $\hmdpd_f$ to satisfy the gap condition \cref{def:iota-gap}.

Utilizing this definition in \cref{lm:concentration-bound} we show that we can obtain a concentration result for a policy $\pi_f^*$ in another MDP $\hmdpd_g$ which has the same empirical transition matrix $\hat{P}$ and reward function $g$.

To prove our results we will utilize \cref{lm:concentration-bound} at every iteration of the primal update in \cref{eq:primal-update} $f = r_p + \blttop \cc$ and $g = r_p$ and the corresponding $\blt$. In that way we will be able to bound $\left\| \Vrp^{\hpit} - \hVrp^{\hpit} \right\|_{\infty}$ and hence $\left| \Vrphpit - \hVrphpit \right|$. To do that we have to ensure that the empirical unconstrained MDP with reward $r_p + \bltop \cc$ satisfy the gap condition in \cref{lm:concentration-bound} for every $\bl \in \Lambda^d$. To that end, we perturb  the rewards in Line 3 of \cref{alg:alg1} and using a union bound over $\Lambda^d$ we prove that that with probability $1- \delta/10$, $\hmdpd_{r_p + \bltop \cc}$ satisfy the gap condition in \cref{def:iota-gap} with $\iota = \frac{\omega \delta (1 - \gamma)}{30 |\Lambda|^d|\cS||\cA|^2}$ for every $\bl \in \Lambda^d$. Therefore we can utilize \cref{lm:concentration-bound} with $f = r_p + \bltop \cc$ for all $t \in [T]$, and $g = r_p$ and for $g = c$.  The concentration result for each policy $\hpi$ and hence for $\bpi_T$ are provided through the theorem below.

\begin{theorem}
    For $\delta \in (0, 1)$, $\omega \leq 1$ and $C(\delta) =  72 \log\left( \frac{16(1 + \omega + dU) |\cS| |\cA| \log(e/1-\gamma)}{(1-\gamma)^2 \iota\delta} \right)$, where $\iota = \frac{\omega \delta (1-\gamma)\epsilon_1^d}{30 U^d |\cS| |\cA|^2}$, if $N \geq \frac{4 C(\delta/d)}{1-\gamma}$, then for $\bpi_T$ output by Alg. \ref{alg:alg1}, with probability at least $1 - \delta/5$,
    \begin{align*}
        \left| \Vrpbpi - \hVrpbpi \right| \,\leq\, 2\sqrt{\frac{C(\delta)}{N(1 - \gamma)^3}};\quad \left| \Vcibpi - \hVcibpi \right| \,\leq\, \sqrt{\frac{C(\delta/d)}{N(1 - \gamma)^3}} \; \fac
    \end{align*}
\end{theorem}

\begin{proof}
    \begin{align*}
    \left| \Vrpbpi - \hVrpbpi \right| \;
        =&\; \left|\frac{1}{T} \sum_{t = 0}^{T - 1} \left[ \Vrphpit - \hVrphpit \right] \right| \;\leq\; \frac{1}{T} \sum_{t=0}^{T-1}\left| \Vrphpit - \hVrphpit \right| \\[0.7ex]
        \leq&\; \frac{1}{T} \sum_{t=0}^{T-1} \left\| V_{r_p}^{\pi_t} - \hVrp^{\pi_t} \right\|_{\infty}
\end{align*}

As we have already discussed in \cref{lm:union-bound-all-cmdps-satifsy-separation} $\hat{M}_{r_p + \blttop \cc}$ satisfies the gap condition with $\iota = \frac{\omega \delta (1 - \gamma)}{30 |\Lambda|^d|\cS||\cA|^2}$ for every $\blt \in \Lambda^d$. Given that $|\Lambda| = \frac{U}{\enet}$, we rewrite $\iota = \frac{\omega \delta (1 - \gamma) \epsilon_1^d}{30 U^d |\cS| |\cA|^2}$. Since $\pi_t \coloneq \argmax_\pi \hV_{r_p + \blttop \cc}^\pi$ we use \cref{lm:concentration-bound} with $f = r_p + \blttop \cc$ (we have that $f_{\max} = 1 + \omega + dU$) and $g = r_p$ and obtain that for $N \geq \frac{4C(\delta)}{1-\gamma}$, for each $t \in [T]$, with probability at least $1 - \delta/10$:
\[
    \left\| V_{r_p}^{\pi_t} - \hVrp^{\pi_t} \right\|_{\infty} \;\leq\; \sqrt{\frac{C(\delta)}{N(1-\gamma)^3}}(1 + \omega) \;\leq\; 2 \sqrt{\frac{C(\delta)}{N(1 - \gamma)^3}}
\]

Therefore with probability at least $1 - \delta/10$:
\begin{equation}
    \left| \Vrpbpi - \hVrpbpi \right| \;\leq\; 2 \sqrt{\frac{C(\delta)}{N(1 - \gamma)^2}} \label{th:thm2-step1}
\end{equation}

Utilizing \cref{lm:concentration-bound} again this time with $f = r_p + \blttop \cc$ and $g = c_i, \, i \in [d]$ we have that for each $c_i$, $i \in [d]$ and with probability $1 - \frac{\delta}{10d}$ it holds that:
\begin{equation}
    \left| \Vcibpi - \hVcibpi \right| \;\leq\; \sqrt{\frac{C(\delta/d)}{N(1 - \gamma)^3}} \label{th:thm2-step2}
\end{equation}

Using a union bound over \ref{th:thm2-step1} and \ref{th:thm2-step2}, we get the desired result with probability at least $1 - \delta/5$.
\end{proof}

\begin{lemma}[Lemma 5 of \citep{vaswani2022near}]\label{lm:concentration-bound}
    Let $f, g$ be two reward functions. We denote $f_{\max} \coloneq \max_{(s, a) \in \cS \times \cA} f(s, a)$. Define $\hpi_f^* \coloneq \arg \max_{\pi} \hV_f^{\pi}$. If (i) $\cE$ is the event that the $\iota$-gap condition in \cref{def:iota-gap} holds for $\hmdpd_f$, and (ii) for $\delta \in (0, 1)$ and $C(\delta) = 72 \log \left(\frac{16 f_{\max} |\cS| |\cA| \log(e/1-\gamma)}{(1-\gamma)^2\iota \delta},\right)$, the number of samples per state-action pair is $N \geq \frac{4 C(\delta)}{1-\gamma}$, then with probability at least $\Pr[\cE] - \delta/10$,
    \begin{equation*}
        \left\| \hV_{g}^{\hpi_f^*} - V_{g}^{\hpi_f^*} \right\|_{\infty} \;\leq\; \sqrt{\frac{C(\delta)}{N (1 - \gamma)^3}} g_{\max}
    \end{equation*}
\end{lemma}

\begin{lemma}[Lemma 15 of \citep{vaswani2022near}]\label{lm:diff-of-perturbed-value-functions-bound}
    For any policy $\pi$, we have 
    \[
        \left\| \Vrpi -  \Vrppi \right\|_{\infty} \;\leq\; \frac{\omega}{1-\gamma};\quad  \left\| \hVrpi -  \hVrppi \right\|_{\infty} \;\leq\; \frac{\omega}{1-\gamma}
    \]
\end{lemma}

\begin{lemma}\label{lm:data-dep-policies-concentration-bounds}
    For any $\delta \in (0, 1)$, $\omega \leq 1$ and $C'(\delta) = 72 \log\left(\frac{4|S| \log(e/1-\gamma)}{\delta} \right)$, if $N \geq \frac{4C'(\nicefrac{\delta}{d})}{1-\gamma}$ and $B(\delta, N) \coloneq  \sqrt{\frac{C'(\delta)}{(1-\gamma)^3N}}$, then for any data-independent policy $\tpi$, we have:
    \begin{enumerate}
        \item $\left| \Vrptpi - \hVrptpi \right| \;\leq\; 2B(\delta, N)$, with probability at least $1 - \delta$.
        \item $\left| \Vcitpi - \hVcitpi \right| \;\leq\; B\left(\frac{\delta}{d}, N\right),\, \fac$, with probability at least $1 - \delta$.
    \end{enumerate}
    Note that if a conjuction of these cases is required, we can use a union bound.
\end{lemma}

\begin{proof}

Since $\tpi$ is fixed and not data dependent policies we can use Lemma 1 of \cite{li2020breaking} for each inequality. In that case it will hold that with probability $1 - \delta$:
\begin{equation}
    \left| \Vrptpi - \hVrptpi \right| \;\leq\; 2B(\delta, N) \label{eq:data-dep-lemma-case1}
\end{equation}

It will also hold that $\fac$ with probability $1 - \frac{\delta}{d}$:
\begin{align*}
    \left| \Vcitpi - \hVcitpi \right| \;\leq\; B\left( \frac{\delta}{d}, N \right)
\end{align*}

Using a union bound over the $d$ constraints we get that with probability $1 - \delta$
\begin{equation}
    \left| \Vcitpi - \hVcitpi \right| \;\leq\; B\left(\frac{\delta}{d}, N\right) \;\, \fac \label{eq:data-dep-lemma-case2}
\end{equation}

Furthermore, with a union bound over \ref{eq:data-dep-lemma-case1} and \ref{eq:data-dep-lemma-case2}, we can also get that both cases hold with probability at least $1 - 2\delta$.

\end{proof}

\begin{lemma}\label{lm:union-bound-all-cmdps-satifsy-separation}
    With probability at least $1 - \delta/10$ for every $\bl \in \Lambda^d$, $\hmdpd_{r_p + \bltop \cc}$ satisfies the gap condition in Definition \ref{def:iota-gap} with $\iota = \frac{\omega \delta (1 - \gamma)}{30 |\Lambda|^d|\cS||\cA|^2}$
\end{lemma}

\begin{proof}
    Using Lemma 18 of \citep{vaswani2022near} with a union bound over $\Lambda^d$ gives the desired result.    
\end{proof}

%% file: appendices/relaxed-feasibility-appendix.tex
\section{Upper Bound under Relaxed Feasibility (proof of \cref{th:theorem-3})}\label{Appendix-C}

\begin{theorem}
    For a fixed $\epsilon \in \left(0, \frac{1}{1 - \gamma}\right]$ and $\delta \in (0, 1)$, Alg. \ref{alg:alg1} with $N \,=\, \cOT \left(\frac{d \log(1/\delta)}{(1-\gamma)^3\epsilon^2} \right)$ samples, $\bb' \,=\, \bb - \frac{3\epsilon}{8}$, $\omega = \frac{\epsilon(1-\gamma)}{8}$, $U \,=\, \cO \left(\frac{1}{\epsilon(1 - \gamma)}\right)$, $\enet = \cO \left( \epsilon^2(1-\gamma)^2 \right)$ and $T = \cO \left(\frac{d^2}{(1-\gamma)^4\epsilon^4}\right)$, returns a policy $\bpit$ that satisfies the objective in Eq. \ref{eq:relaxed-CMDP-formulation} with probability at least $1 - 3\delta$.
\end{theorem}

\begin{proof}

We will prove the result for a general primal-dual error $\eopt \,<\, \epsilon$ and $\bb' \,=\, \bb - \frac{\epsilon - \eopt}{2}$ which will be specified later. In \cref{lm:decomposing-relaxed-suboptimality} we showed that if the constraint value functions are sufficiently concentrated for the optimal policy $\pis$ in $\mdpd$ and the mixture policy $\bpit$ returned by the Algorithm \ref{alg:alg1}, i.e., if
\begin{equation}
    \left| \Vcibpi - \hVcibpi \right| \,\leq\, \frac{\epsilon - \eopt}{2};\quad \left| \Vcipis - \hVcipis \right| \,\leq\, \frac{\epsilon - \eopt}{2} \;\, \fac \label{eq:theorem-relaxed-lemma-conditions}
\end{equation}

then (i) policy $\bpit$ violates the constraint in $\mdpd$ by at most $\epsilon$, i.e., $\Vcbpi \,\geq\, \bb - \epsilon$, and (ii) its suboptimality in $\mdpd$ can be decomposed as:
\begin{equation}
    \Vrpis - \Vrbpi \;\leq\; \frac{2 \omega}{1 - \gamma} + \eopt + \left|\Vrppis - \hVrppis \right| + \left| \Vrpbpi - \hVrpbpi \right| \label{eq:theorem-relaxed-decomposition}
\end{equation}

In order to instantiate the primal-dual algorithm, we still need to specify $U$. To do this, we will use case i) of \cref{lm:dual-variable-bound}. To apply this lemma, we need a concentration result on $\left| \Vcitpi - \Vcitpi \right| \, \fac$ for a feasible policy $\tpi$. Since $\pis$ is a feasible policy, we require $\left| \Vcipis - \Vcipis \right| \,\leq\, \frac{\epsilon - \eopt}{4}\, \fac$. \cref{lm:dual-variable-bound} shows that when this concentration result holds, then we can upper bound the infinity norm of the optimal dual variable $\blsnorm$ by $\frac{4(1 + \omega)}{(\epsilon - \eopt)(1 - \gamma)}$.

Given the above result we can define all the parameters of the algorithm except the number of samples required for each state-action $N$. We set $\eopt = \frac{\epsilon}{4}$ and therefore $\bb' = \bb - \frac{3\epsilon}{8}$ and $\omega = \frac{\epsilon(1 - \gamma)}{8} < 1$. Now we are able to compute $U$ by expanding the upper-bound on $\blsnorm$ mentioned in the previous paragraph:
\[
    \blsnorm \,\leq\, \frac{4 \left( 1 + \frac{\epsilon (1-\gamma)}{8} \right)}{\frac{3\epsilon}{4}(1-\gamma)} \,=\, \frac{16 + 2\epsilon (1-\gamma)}{3\epsilon(1-\gamma)} \,\overset{\epsilon \in (0, \nicefrac{1}{1-\gamma}]}{\leq}\, \frac{18}{3\epsilon(1-\gamma)} \,<\, \frac{8}{\epsilon(1-\gamma)}
\]

Therefore, by setting $C = \frac{8}{\epsilon(1-\gamma)}$ and $U = 2C$, we ensure that $\blsnorm < U$ and $U - \blsnorm > C$.

To guarantee that the primal-dual algorithm outputs an $\frac{\epsilon}{4}$-approximate policy we utilize \cref{th:theorem-1} in order to specify the complexity of $T$ and $\enet$. In more detail based on \cref{th:theorem-1} and for $U = \frac{16}{\epsilon (1-\gamma)}$, $\eopt = \frac{\epsilon}{4}$:
\begin{align*}
    T \,
        =&\, \frac{4 U^2 d^2}{\eopt^2 (1 - \gamma)^2} \left[ 1 + \frac{1}{(U - \blsnorm)^2} \right] \,=\, \frac{64 d^2}{\epsilon^2 (1 - \gamma)^2}\left[ U^2 + \left(\frac{U}{U - \blsnorm} \right)^2 \right] \\[1ex]
        \leq&\, \frac{256 d^2}{\epsilon^2 (1 - \gamma)^2} \left[C^2 +1\right] \,<\, \frac{512 d^2}{\epsilon^2(1-\gamma)^2}C^2 \,=\, \frac{512 d^2}{\epsilon^2(1-\gamma)^2} \cdot \frac{64}{\epsilon^2(1-\gamma)^2}
\end{align*}

Thus we conclude that $T = \cO \left( \nicefrac{d^2}{\epsilon^4 (1 - \gamma)^4}\right)$. Using \cref{th:theorem-1} we also set $\enet$:
\[
    \enet \,=\, \frac{\eopt^2 (1 - \gamma)^2 (U - \blsnorm)}{6dU} \,=\, \frac{\epsilon^2(1-\gamma)^2}{96d} \cdot \frac{U - \blsnorm}{U} \,\leq\, \frac{\epsilon^2(1-\gamma)^2}{96d}
\]

Thus we conclude that $\enet \,=\, \cO \left(\frac{\epsilon^2 (1 - \gamma)^2}{d} \right)$.

Using the fact that $\omega = \frac{\epsilon(1-\gamma)}{8}$ and $\eopt = \frac{\epsilon}{4}$, we now proceed to simplify Eq. \ref{eq:theorem-relaxed-decomposition} as:
\[
    \Vrpis - \Vrbpi \,\leq\, \frac{\epsilon}{2} + \left| \Vrppis - \hVrppis \right| + \left| \Vrpbpi - \hVrpbpi \right|
\]

Combining the previous results and in order to guarantee an $\epsilon$-reward suboptimality for $\bpit$ we require
\begin{align}
    &\left| \Vcibpi - \hVcibpi \right| \,\leq\, \frac{3 \epsilon}{8} \;\, \fac;\quad
    \left| \Vcipis - \hVcipis \right| \,\leq\, \frac{3 \epsilon}{16} \;\, \fac; \nonumber \\[1ex]
    &\left| \Vrppis - \hVrppis \right| \,\leq\, \frac{\epsilon}{4};\qquad\qquad\quad\; \left| \Vrpbpi - \hVrpbpi \right| \,\leq\, \frac{\epsilon}{4} \label{eq:relaxed-feasibility-requirements}
\end{align}

To bound the concentration terms for $\bpit$ in \cref{eq:relaxed-feasibility-requirements} we utilize \cref{th:theorem-2} with $\omega = \frac{\epsilon (1 - \gamma)}{8}$, $U = \frac{16}{\epsilon (1-\gamma)}$ and $\enet = \frac{\epsilon^2 (1 - \gamma)^2}{d}$. In this case for $\iota$ we get:
\begin{align*}
    \iota \;=\; \frac{\omega \delta (1 - \gamma) \epsilon_1^d}{30 U^d |\cS| |\cA|^2} \;=\; \frac{\epsilon^{3d+1} (1-\gamma)^{3d+2} \delta}{240 (1536)^d d^d |\cS| |\cA|^2} \;=\; \left(\frac{1}{1536}\right)^d \frac{\epsilon^{3d+1} (1 - \gamma)^{3d+2} \delta}{240 d^d |\cS| |\cA|^2}
\end{align*}

We now calculate $C(\delta)$ as:
\begin{align*}
    C(\delta) \;
        =&\; 72 \log \left( \frac{16 (1 + \omega + dU) |\cS| |\cA| \log(e / 1-\gamma)}{(1 - \gamma)^2 \iota \delta}\right)  \\[1ex]
        =&\; 72 \log \left(3840 \cdot 1536^d \frac{d^d |\cS|^2 |\cA|^3 \left(1 + \frac{\epsilon(1-\gamma)}{8}+\frac{16d}{\epsilon(1-\gamma)} \right)}{(1-\gamma)^{3d+5}\epsilon^{3d+1}\delta^2} \right) \tag{substitute value for $i$} \\[1ex]
        =&\; 72 \log \left(3840 \cdot 1536^d \frac{d^d |\cS|^2 |\cA|^3 \left( \frac{8\epsilon(1-\gamma) + \epsilon^2(1-\gamma)^2 + 128d}{8\epsilon(1-\gamma)} \right)}{(1-\gamma)^{3d+5}\epsilon^{3d+1}\delta^2} \right) \tag{simplify inner parenthesis} \\[1ex]
        =&\; \cO \left( 72 \log \left(3840 \cdot 1536^d \frac{d^d |\cS|^2 |\cA|^3 \frac{d}{\epsilon(1 - \gamma)}}{(1-\gamma)^{3d+5}\epsilon^{3d+1}\delta^2} \right) \right) \tag{inner parenthesis $= \cO \left( \frac{d}{\epsilon(1 - \gamma)} \right)$} \\[1ex]
        =&\; \cO \left( \log \left( 1536^d \frac{d^{d+1}|\cS|^2|\cA|^3}{(1-\gamma)^{3d+6}\epsilon^{3d+2}\delta^2} \right) \right) \tag{drop constants, simplify} \\[1ex]
        =&\; \cO \left( d\log \left( \frac{d |\cS|^2 |\cA|^3}{(1-\gamma)\epsilon \delta}\right) \right) \tag{logarithmic power rule, drop constants}
\end{align*}

Having defined the complexity of $C(\delta)$ in order to satisfy the concentration bound for $\bpit$ we require that:
\begin{equation}
    2\sqrt{\frac{C(\delta)}{N(1-\gamma)^3}} \;\leq\; \frac{\epsilon}{4} \;\;\Rightarrow\;\; N \,\geq\, \cO \left(\frac{C(\delta)}{(1-\gamma)^3\epsilon^2} \right) \,=\, \cOT \left( \frac{d \log\left(\frac{1}{\delta} \right)}{(1-\gamma)^3\epsilon^2} \right) \label{eq:data-dep-sample-compl}
\end{equation}

To bound the remaining concentration terms $\pis$ in \cref{eq:relaxed-feasibility-requirements} we use \cref{lm:data-dep-policies-concentration-bounds} with $\tpi = \pis$. In that case for $C'(\delta) = 72 \log\left(\frac{4|S| \log(e/1-\gamma)}{\delta} \right)$, we require that:
\begin{align}
    2\sqrt{\frac{C'(\delta)}{N (1 - \gamma)^3}} \;\leq\; \frac{\epsilon}{4}; &\quad\; \sqrt{\frac{C'\left( \frac{\delta}{d} \right)}{N (1 - \gamma)^3}} \;\leq\; \frac{3\epsilon}{16} \nonumber \\[1.5ex]
    \Rightarrow\;\, N \,\geq\, \cO \left(\frac{C'(\delta)}{(1-\gamma)^3\epsilon^2} \right); &\quad\; N \,\geq\, \cO \left(\frac{C'\left( \frac{\delta}{d} \right)}{(1-\gamma)^3\epsilon^2} \right) \nonumber \\[1.5ex]
    \Rightarrow\;\, N \,\geq\, \cOT &\left( \frac{\log \left( \frac{d}{\delta} \right)}{(1-\gamma)^3\epsilon^2}\right) \label{eq:data-ind-sample-compl}
\end{align}

Combining the results for the data-dependent policy $\bpit$ in \cref{eq:data-dep-sample-compl} that hold with probability $1 - \delta/5$, and the results for the data-independent policy $\pis$ in \cref{eq:data-ind-sample-compl} that hold with probability $1 - 2\delta$ (using a union bound), taking the loosest sample complexity for $N$, which is $\cOT \left( \frac{d \log\left(\frac{1}{\delta} \right)}{(1-\gamma)^3\epsilon^2} \right)$ and then using a union bound, we finally get that the bounds in \cref{eq:relaxed-feasibility-requirements} are satisfied with probability at least $1 - (2\delta + \delta/5) \,\geq\, 1 - 3\delta$, which completes the proof.    
\end{proof}

\begin{lemma}[Decomposing the suboptimality]\label{lm:decomposing-relaxed-suboptimality}
    For $b_i' = b_i - \frac{\epsilon - \eopt}{2}, \, \fac$, if (i) $\eopt \,<\, \epsilon$, and (ii) the following conditions are satisfied $\fac$,
    \begin{equation*}
        \left| \Vcibpi - \hVcibpi \right| \,\leq\, \frac{\epsilon - \eopt}{2};\quad \left| \Vcipis - \hVcipis \right| \,\leq\, \frac{\epsilon - \eopt}{2}
    \end{equation*}
    then (a) the policy $\bpit$ violates the constraints by at most $\epsilon$ i.e. $\Vcibpi \geq b_i - \epsilon, \;\, \fac$ and (b) its optimality gap can be bounded as:
    \begin{equation*}
        \Vrpis - \Vrbpi \,\leq\, \frac{2 \omega}{1 - \gamma} + \eopt + \left| \Vrppis - \hVrppis \right| + \left| \Vrpbpi - \hVrpbpi \right|
    \end{equation*}
\end{lemma}

\begin{proof}
    Based on the Theorem \ref{th:theorem-1}, we know that $\fac$:
\begin{align}
    &\hVcibpi \;\geq\; b_i' - \eopt \nonumber \\[0.6ex]
    &\Rightarrow\;\, \Vcibpi \;\geq\; \Vcibpi - \hVcibpi + b_i' - \eopt \nonumber \\[0.6ex]
    &\Rightarrow\;\, \Vcibpi \;\geq\; -\left| \Vcibpi - \hVcibpi \right| + b_i' - \eopt \label{eq:relaxed-decomposition-constraint-cond1}
\end{align}

Since we allow $\bpi_T$ to violate the constraints in the true CMDP by at most $\epsilon$ (see \cref{eq:relaxed-CMDP-formulation}), we require $V_{c_i}^{\bpi_T}(\rho) \,\geq\, b_i - \epsilon,\; \fac$. Using \cref{eq:relaxed-decomposition-constraint-cond1}, a sufficient condition to ensure that is $\fac$:
\begin{align*}
    -&\left| \Vcibpi - \hVcibpi \right| + b_i' - \eopt \,\geq\, b_i - \epsilon \;\;\;\Leftrightarrow \\[0.6ex]
    &\left| \Vcibpi - \hVcibpi \right| \,\leq\, (b_i' - b_i) - \eopt + \epsilon
\end{align*}

Using that $b_i' = b_i - \frac{\epsilon - \eopt}{2}, \, \forall i \in [d]$ we see that this sufficient condition holds, by the assumption that $\fac: \left| \Vcibpi - \hVcibpi \right| \leq \frac{\epsilon - \eopt}{2}$

Let $\pis$ be the solution to Eq. \ref{eq:CMDP-formulation}. We want to show that $\pis$ is feasible for the constraint problem in Eq. \ref{eq:emprical-CMDP-formulation} e.g. $\hVcipis \geq b_i' \, \fac$. We have:
\begin{align*}
    &\Vcipis \,\geq\, b_i \tag{since $\pis$ is feasible} \\[1ex]
    &\Rightarrow\;\, \hVcipis \,\geq\, b_i + \hVcipis - \Vcipis \tag{add $\hVcipis$ and subtract $\Vcipis$ from both sides} \\[1ex]
    &\Rightarrow\;\, \hVcipis \;\geq\; b_i - \left| \Vcipis - \hVcipis \right| \tag{lower bound $x$ by $-|x|$}
\end{align*}

Since we require $\hVcipis \,\geq\, b_i' \;\, \fac$, utilizing the above equation, a sufficient condition to ensure this is:

\begin{align*}
    b_i - &\left| V_{c_i}^{\pi^*}(\rho) - \hV_{c_i}^{\pi^*}(\rho) \right| \,\geq\, b_i' \;\;\;\Leftrightarrow \\[0.6ex]
    &\left| V_{c_i}^{\pi^*}(\rho) - \hV_{c_i}^{\pi^*}(\rho) \right| \,\leq\, b_i - b_i'
\end{align*}

Since $b_i' = b_i - \frac{\epsilon - \eopt}{2}, \, \fac$, we require that
\[
    \left| \Vcipis - \hVcipis \right| \,\leq\, \frac{\epsilon - \eopt}{2}
\]

Given that the above statements holds, we can decompose the suboptimality in the reward value function as follows:
\begin{align*}
    &\Vrpis - \Vrbpi \\[1ex]
    &=\; \Vrpis - \Vrppis + \Vrppis - \Vrbpi \\[1ex]
    &=\; \left[ \Vrpis - \Vrppis \right] + \Vrppis - \hVrppis + \hVrppis - \Vrbpi \\[1ex]
    &\leq\; \left[ \Vrpis - \Vrppis \right] + \left[\Vrppis - \hVrppis \right] + \hVs - \Vrbpi \\[1ex]
    &\; (\text{By optimality of } \hpi^* \text{and since we have ensured that } \pi^* \text{is feasible for Eq. \ref{eq:emprical-CMDP-formulation}}) \\[1ex]
    &=\; \left[ \Vrpis - \Vrppis \right] + \left[ \Vrppis - \hVrppis \right] + \left[ \hVs - \hVrpbpi \right] + \hVrpbpi - \Vrbpi \\[1ex]
    &=\; \underbrace{\left[ \Vrpis - \Vrppis \right]}_{\text{Perturbation Error}} + \underbrace{\left[ \Vrppis - \hVrppis \right]}_{\text{Concentration Error}} + \underbrace{\left[ \hVs - \hVrpbpi \right]}_{\text{Primal-Dual Error}} + \underbrace{\left[\hVrpbpi - \Vrpbpi \right]}_{\text{Concentration Error}} \\[1ex]
    & \:\,+ \underbrace{\left[\Vrpbpi - \Vrbpi \right]}_{\text{Perturbation Error}}
\end{align*}

For a perturbation magnitude equal to $\omega$, we use \cref{lm:concentration-bound} to bound both perturbation errors by $\frac{\omega}{1-\gamma}$. Using Theorem 1 to bound primal dual error by $\eopt$:

\[
    \Vrpis - \Vrbpi \,\leq\, \frac{2\omega}{1-\gamma} + \eopt + \left[ \Vrppis - \hVrppis \right] + \left[ \hVrpbpi - \Vrpbpi \right]
\]
\end{proof}

\begin{lemma}[Bounding the dual variables]\label{lm:dual-variable-bound}
    Let $\tpi$ be a feasible policy. Define the margin as $\slaterctpi \coloneq \min_{i \in [d]}\{ \Vcitpi - b_i \} \geq 0$ (since $\tpi$ is feasible). We consider two cases: (1) if $\bb' = \bb - \epsilon'$ for $\epsilon' > 0$ and event $\cE_1 = \left\{ \left| \hVcitpi - \Vcitpi \right| \,\leq\, \frac{\epsilon'}{2}, \fac \right\}$ holds, then $\blsnorm \leq \frac{2(1 + \omega)}{\epsilon' (1 - \gamma)}$, and (2): if $\tpi$ is strictly feasible (i.e. $\slaterctpi > 0$), and $\bb' = \bb + \Delta$ for $\Delta \in \left(0, \frac{\slaterctpi}{2} \right)$ and event $\cE_2 = \left\{ \left| \hVcitpi - \Vcitpi \right| \,\leq\, \frac{\slaterctpi}{2} - \Delta, \fac \right\}$ holds, then: $\blsnorm \leq \frac{2(1 + \omega)}{\slaterctpi (1 - \gamma)}$.
\end{lemma}

\begin{proof}

Writing the empirical CMDP in Eq. \ref{eq:emprical-CMDP-formulation} in its Lagrangian form:
\begin{flalign*}
& \hspace{7mm} \hVs \;=\; \max_\pi \min_{\li \geq 0} \hVrppi + \bltop \left[ \hVcpi - \bb' \right] &
\end{flalign*}

Using the results of \cite{paternain2019constrained}, we can switch the min and the max to get:
\begin{flalign*}
& \hspace{17.5mm} \;=\; \min_{\li \geq 0} \max_\pi \hVrppi + \bltop \left[ \hVcpi - \bb' \right] &
\end{flalign*}

\noindent Since $\bls$ is the optimal dual variable for the empirical CMDP in Eq. \ref{eq:emprical-CMDP-formulation}:
\begin{flalign}
& \hspace{17.5mm} \;=\; \max_\pi \hVrppi + \blstop \left[ \hVcpi - \bb' \right] & \nonumber \\[0.6ex]
& \hspace{17.5mm} \;=\; \max_\pi \hVrppi + \sum_{i=1}^d \lis \left[ \hVcipi - b_i' \right] & \nonumber \\[0.6ex]
& \hspace{17.5mm} \;\geq\; \hVrptpi + \sum_{i=1}^d \lis \left[ \hVcitpi - b_i' \right] & \tag{since the above term is a max over all policies} \\[0.6ex]
& \hspace{17.5mm} \;=\; \hVrptpi + \sum_{i=1}^d \lis \left[ \left(\hVcitpi - b_i\right) + \left(b_i - b_i'\right) \right] & \tag{$\pm b_i$} \nonumber \\[0.6ex]
& \hspace{17.5mm} \;\geq\; \hVrptpi + \sum_{i=1}^d \lis \left[ \hVcitpi - \Vcitpi + \slaterctpi + (b_i - b_i') \right] & \label{eq:dual-var-bound-lemma-common-ineq}
\end{flalign}
where the last inequality follows from the definition of $\slaterctpi$, i.e. $0 < \slaterctpi \leq \Vcitpi - b_i,\, \fac$.

1) For the first case, we continue from \cref{eq:dual-var-bound-lemma-common-ineq}:
\begin{flalign}
    & \hspace{7mm} \hVs \;\geq\; \hVrptpi + \sum_{i=1}^d \lis \left[ \hVcitpi - \Vcitpi + \slaterctpi + (b_i - b_i') \right] & \nonumber \\[0.6ex]
    & \hspace{17.5mm} \;=\; \hVrptpi + \sum_{i=1}^d \lis \left[ \hVcitpi - \Vcitpi + \slaterctpi + \epsilon' \right] & \tag{from the i) lemma assumption} \nonumber \\[0.6ex]
    & \hspace{17.5mm} \;\geq\; \hVrptpi + \sum_{i=1}^d \lis \left[ \hVcitpi - \Vcitpi + \epsilon' \right] & \tag{since $\lis \geq 0 \;\fac,\, \slaterctpi \geq 0$} \nonumber \\[0.6ex]
    & \hspace{17.5mm} \;\geq\; \hVrptpi + \sum_{i=1}^d \lis \left[ - \left| \hVcitpi - \Vcitpi \right| + \epsilon' \right] & \tag{$x \geq -|x|,\; \forall x \in \R$} \nonumber \\[0.6ex]
    & \hspace{17.5mm} \;\geq\; \hVrptpi + \sum_{i=1}^d \lis \left[ - \frac{\epsilon'}{2} + \epsilon' \right] & \tag{from event $\cE_1$ in lemma assumptions} \nonumber \\[0.6ex]
    & \hspace{17.5mm} \;\geq\; \hVrptpi + \frac{\epsilon'}{2}\sum_{i=1}^d \lis & \label{eq:dual-var-bound-case1}
\end{flalign}

We now rewrite \ref{eq:dual-var-bound-case1}:
\begin{align*}
    \frac{\epsilon'}{2}\sum_{i=1}^d \lis \;\leq&\; \left[ \hVrptpi - \hVrptpi \right] \\[0.6ex]
        \leq&\; \frac{1 + \omega}{1 - \gamma}  \tag{because $\hVrptpi, \hVrptpi \in \left[0, \frac{1 + \omega}{1 - \gamma} \right]$} \\[0.6ex]
        \Rightarrow\;\;\;&\; \sum_{i=1}^d \lis \;\leq\; \frac{2(1 + \omega)}{\epsilon'(1 - \gamma)} \\[0.6ex]
        \Rightarrow\;\;\;&\; \blsnorm \;\leq\; \frac{2(1 + \omega)}{\epsilon'(1 - \gamma)} \tag{since $\li \,\geq\, 0,\; \fac$}
\end{align*}

2) For the second case, we again continue from \cref{eq:dual-var-bound-lemma-common-ineq}:
\begin{flalign}
    & \hspace{7mm} \hVs \;\geq\; \hVrptpi + \sum_{i=1}^d \lis \left[ \hVcitpi - \Vcitpi + \slaterctpi + (b_i - b_i') \right] & \nonumber \\[0.6ex]
    & \hspace{17.5mm} \;=\; \hVrptpi + \sum_{i=1}^d \lis \left[ \hVcitpi - \Vcitpi + \slaterctpi - \Delta \right] & \tag{from the ii) lemma assumption} \nonumber \\[0.6ex]
    & \hspace{17.5mm} \;=\; \hVrptpi + \sum_{i=1}^d \lis \left[ \slaterctpi - \Delta - \left| \hVcitpi - \Vcitpi \right| \right] & \tag{$x \geq -|x|,\; \forall x \in \R$} \nonumber \\[0.6ex]
    & \hspace{17.5mm} \;=\; \hVrptpi + \sum_{i=1}^d \lis \left[ \frac{\slaterctpi}{2} \right] & \tag{from event $\cE_2$ in lemma assumptions} \nonumber \\[0.6ex]
    & \hspace{17.5mm} \;\geq\; \hVrptpi + \frac{\slaterctpi}{2}\sum_{i=1}^d \lis \label{eq:dual-var-bound-case2}
\end{flalign}

We now rewrite \ref{eq:dual-var-bound-case2}:
\begin{align*}
    \frac{\slaterctpi}{2}\sum_{i=1}^d \lis \;\leq&\; \left[ \hVrptpi - \hVrptpi \right] \\[0.6ex]
        \leq&\; \frac{1 + \omega}{1 - \gamma}  \tag{because $\hVrptpi, \hVrptpi \in \left[0, \frac{1 + \omega}{1 - \gamma} \right]$} \\[0.6ex]
        \Rightarrow\;\;\;&\; \sum_{i=1}^d \lis \;\leq\; \frac{2(1 + \omega)}{\slaterctpi(1 - \gamma)} \tag{$\slaterctpi > 0$ because $\tpi$ is strictly feasible} \\[0.6ex]
        \Rightarrow\;\;\;&\; \blsnorm \;\leq\; \frac{2(1 + \omega)}{\slaterctpi(1 - \gamma)} \tag{since $\li \,\geq\, 0,\; \fac$}
\end{align*}

\end{proof}

%% file: appendices/strict_feasibility-appendix.tex
\section{Upper Bound under Strict Feasibility (proof of \cref{th:theorem-4})}\label{Appendix-D}

\begin{theorem}
    For a fixed $\epsilon \in \left(0, \frac{1}{1 - \gamma}\right]$ and $\delta \in (0, 1)$, Alg. \ref{alg:alg1} with $N \,=\, \cOT \left(\frac{d^3 \log(1/\delta)}{(1-\gamma)^5 \epsilon^2 {\slatercs}^2} \right)$ samples, $\bb' \,=\, \bb + \frac{\epsilon(1-\gamma)\slatercs}{20}$, $\omega = \frac{\epsilon(1-\gamma)}{10}$, $U \,=\, \frac{4(1+\omega)}{\slatercs(1-\gamma)}$, $\enet = \cO \left( \epsilon^2 (1-\gamma)^4 {\slatercs}^2 \right)$ and $T = \cO \left(\frac{d^4}{(1-\gamma)^6 \epsilon^2 {\slatercs}^2 }\right)$, returns a policy $\bpit$ that satisfies the objective in Eq. \ref{eq:strict-CMDP-formulation} with probability at least $1 - 4\delta$.
\end{theorem}

\begin{proof}
We will prove the result for a general $\bb' \,=\, \bb + \Delta$ for $\Delta > 0$ which will be specified later. In \cref{lm:decomposing-suboptimality-strict} we showed that if the constraint value functions are sufficiently concentrated for the optimal policy $\pis$ in $\mdpd$ and the mixture policy $\bpit$ returned by the Algorithm \ref{alg:alg1}, i.e., if
\begin{equation}
    \left| \Vcibpi - \hVcibpi \right| \,\leq\, \Delta - \eopt;\quad \left| \Vcipis - \hVcipis \right| \,\leq\, \Delta \;\, \fac \label{eq:theorem-strict-lemma-conditions}
\end{equation}

then (i) policy $\bpit$ satisfies all the constraints in $\mdpd$, i.e., $\Vcbpi \,\geq\, \bb$, and (ii) its suboptimality in $\mdpd$ can be decomposed as:
\begin{equation}
    \Vrpis - \Vrbpi \;\leq\; \frac{2\omega}{1-\gamma} + \eopt + 2d\lambda_{\min}^*\Delta + \left[ \Vrppis - \hVrppis \right] + \left[ \Vrpbpi - \Vrpbpi \right] \label{eq:theorem-strict-decomposition}
\end{equation}

In order to instantiate the primal-dual algorithm, we still need to specify $U$. To do this, we will use case ii) of \cref{lm:dual-variable-bound}. To apply this lemma, we need a concentration result on $\left| \Vcitpi - \Vcitpi \right| \, \fac$ for a strictly feasible policy $\tpi$. Since $\pics$ is a strictly feasible policy, we require the bound $\left| \Vcipics - \Vcipics \right| \,\leq\, \frac{\slatercs}{2} - \Delta\, \fac$. \cref{lm:dual-variable-bound} shows that when this concentration result holds, then we can upper bound the infinity norm of the optimal dual variable $\blsnorm$ by $\frac{2(1 + \omega)}{\slatercs(1 - \gamma)}$.

Given the above result we can define all the parameters of the algorithm except the number of samples required for each state-action $N$. We set $\eopt = \frac{\Delta}{5}$, $\bb' = \bb + \Delta$ and $\omega = \frac{\epsilon(1 - \gamma)}{10} < 1$. Now we are able to compute $U$ by expanding the upper-bound on $\blsnorm$ mentioned in the previous paragraph:
\[
    \blsnorm \,\leq\, \frac{2 \left( 1 + \frac{\epsilon (1-\gamma)}{10} \right)}{\slatercs (1-\gamma)} \,=\, \frac{2 + \frac{\epsilon (1-\gamma)}{5}}{\slatercs (1-\gamma)} \,\overset{\epsilon \in (0, \nicefrac{1}{1-\gamma}]}{\leq}\, \frac{2 + \frac{1}{5}}{\slatercs (1-\gamma)} \,<\, \frac{4}{\slatercs (1-\gamma)}
\]

Therefore, by setting $C = \frac{4}{\slatercs (1-\gamma)}$ and $U = 2C$, we ensure that $\blsnorm < U$ and $U - \blsnorm > C$.

To guarantee that the primal-dual algorithm outputs an $\frac{\Delta}{5}$-approximate policy we utilize \cref{th:theorem-1} in order to specify the complexity of $T$ and $\enet$. In more detail based on \cref{th:theorem-1} and for $U = \frac{8}{\slatercs (1-\gamma)}$, $\Delta = \frac{\epsilon (1-\gamma) \slatercs}{40d}$, $\eopt = \frac{\Delta}{5} = \frac{\epsilon (1-\gamma) \slatercs}{200d}$:
\begin{align*}
    T \,
        =&\, \frac{4 U^2 d^2}{\eopt^2 (1 - \gamma)^2} \left[ 1 + \frac{1}{(U - \blsnorm)^2} \right] \,=\, \frac{100 d^2}{\Delta^2 (1 - \gamma)^2 }\left[ U^2 + \left(\frac{U}{U - \blsnorm} \right)^2 \right] \\[1ex]
        \leq&\, \frac{400 d^2}{\Delta^2 (1 - \gamma)^2} \left[C^2 +1\right] \,<\, \frac{800 d^2}{\Delta^2 (1-\gamma)^2}C^2 \,=\, \frac{800 d^2}{\Delta^2 (1-\gamma)^2} \cdot \frac{16}{{\slatercs}^2 (1-\gamma)^2}
\end{align*}

Thus we conclude that $T = \cO \left( \nicefrac{d^4}{\epsilon^2 {\slatercs}^4 (1-\gamma)^6} \right)$. Using \cref{th:theorem-1} we also set $\enet$:
\[
    \enet \,=\, \frac{\Delta^2 (1 - \gamma)^2 (U - \blsnorm)}{6dU} \,=\, \frac{\Delta^2 (1-\gamma)^2}{150d} \cdot \frac{U - \blsnorm}{U} \,\leq\, \frac{\epsilon^2 (1-\gamma)^4 {\slatercs}^2}{240000 d}
\]

Thus we conclude that $\enet \,=\, \cO \left( \frac{\epsilon^2 (1 - \gamma)^4 {\slatercs}^2}{d} \right)$.

Using the fact that $\omega = \frac{\epsilon(1-\gamma)}{10}$, $\eopt = \frac{\Delta}{5}$ and $\lambda_{\min}^* \leq \blsnorm < \frac{2(1+\omega)}{\slatercs(1-\gamma)}$ we can simplify Eq. \ref{eq:theorem-strict-decomposition} as:
\begin{align*}
    \Vrpis - \Vrbpi \,\leq&\, \frac{2\omega}{1-\gamma} + \eopt + 2d\lambda_{\min}^*\Delta + \left| \Vrppis - \hVrppis \right| + \left| \Vrpbpi - \hVrpbpi \right| \\[1ex]
    \leq&\, \frac{\epsilon}{5} + \frac{\Delta}{5} + 2d \frac{2(1+\omega)}{\slatercs(1-\gamma)}\frac{\epsilon(1-\gamma)\slatercs}{40d} + \left| \Vrppis - \hVrppis \right| + \left| \Vrpbpi - \hVrpbpi \right| \\[1ex]
    \leq&\, \frac{\epsilon}{5} + \frac{\Delta}{5} + \frac{(1+1)\epsilon}{10} + \left| \Vrppis - \hVrppis \right| + \left| \Vrpbpi - \hVrpbpi \right| \tag{$\omega \leq 1$} \\[1ex]
    \leq&\, \frac{3\epsilon}{5} + \left| \Vrppis - \hVrppis \right| + \left| \Vrpbpi - \hVrpbpi \right| \tag{$\Delta \leq \epsilon$}
\end{align*}

Combining the previous results and in order to guarantee an $\epsilon$-reward suboptimality for $\bpit$ we require
\begin{align}
    &\left| \Vcibpi - \hVcibpi \right| \,\leq\, \frac{4 \Delta}{5} \;\, \fac;\quad
    \left| \Vcipis - \hVcipis \right| \,\leq\, \Delta \;\, \fac; \nonumber \\[1ex]
    &\left| \Vcipics - \hVcipics \right| \,\leq\, \frac{19 \Delta}{5}; \;\, \fac \nonumber \\[1ex]
    &\left| \Vrppis - \hVrppis \right| \,\leq\, \frac{\epsilon}{5};\qquad\qquad\quad\;\;\, \left| \Vrpbpi - \hVrpbpi \right| \,\leq\, \frac{\epsilon}{5} \label{eq:strict-feasibility-requirements}
\end{align}

To bound the concentration terms for $\bpit$ in \cref{eq:strict-feasibility-requirements} we utilize \cref{th:theorem-2} with $\omega = \frac{\epsilon (1 - \gamma)}{10}$, $U = \frac{8}{\slatercs (1-\gamma)}$ and $\enet = \frac{\epsilon^2 (1 - \gamma)^4 {\slatercs}^2}{240000d}$. In this case for $\iota$ we get:
\begin{align*}
    \iota \;=\; \frac{\omega \delta (1 - \gamma) \epsilon_1^d}{30 U^d |\cS| |\cA|^2} \;=\; \frac{\epsilon^{2d+1} (1-\gamma)^{5d+2} {\slatercs}^{3d} \delta}{300 (1920000)^d d^d |\cS| |\cA|^2} \;=\; \left(\frac{1}{1920000}\right)^d \frac{\epsilon^{2d+1} (1 - \gamma)^{5d+2} {\slatercs}^{3d} \delta}{300 d^d |\cS| |\cA|^2}
\end{align*}

We now calculate $C(\delta)$ as:
\begin{align*}
    C(\delta) \;
        =&\; 72 \log \left( \frac{16 (1 + \omega + dU) |\cS| |\cA| \log(e / 1-\gamma)}{(1 - \gamma)^2 \iota \delta}\right)  \\[1ex]
        =&\; 72 \log \left(4800 \cdot 1920000^d \frac{d^d |\cS|^2 |\cA|^3 \left(1 + \frac{\epsilon(1-\gamma)}{10}+\frac{8d}{\slatercs(1-\gamma)} \right)}{(1-\gamma)^{5d+5}\,\epsilon^{2d+1}\,{\slatercs}^{3d}\,\delta^2} \right) \tag{substitute value for $i$} \\[1ex]
        =&\; 72 \log \left(4800 \cdot 1920000^d \frac{d^d |\cS|^2 |\cA|^3 \left( \frac{10\slatercs(1-\gamma) + \epsilon\slatercs(1-\gamma)^2 + 80d}{10\slatercs(1-\gamma)} \right)}{(1-\gamma)^{5d+5}\,\epsilon^{2d+1}\,{\slatercs}^{3d}\,\delta^2} \right) \tag{simplify inner parenthesis} \\[1ex]
        =&\; \cO \left( 72 \log \left(4800 \cdot 1920000^d \frac{d^d |\cS|^2 |\cA|^3 \frac{d}{\slatercs(1 - \gamma)}}{(1-\gamma)^{5d+5}\,\epsilon^{2d+1}\,{\slatercs}^{3d}\,\delta^2} \right) \right) \tag{inner parenthesis $= \cO \left( \frac{d}{\slatercs(1 - \gamma)} \right)$} \\[1ex]
        =&\; \cO \left( \log \left( 1920000^d \frac{d^{d+1} |\cS|^2 |\cA|^3}{(1-\gamma)^{5d+6}\,\epsilon^{2d+1}\,{\slatercs}^{3d+1}\,\delta^2} \right) \right) \tag{drop constants, simplify} \\[1ex]
        =&\; \cO \left( d\log \left( \frac{d |\cS|^2 |\cA|^3}{(1-\gamma) \slatercs \epsilon \delta}\right) \right) \tag{logarithmic power rule, drop constants}
\end{align*}

Having defined the complexity of $C(\delta)$ in order to satisfy the concentration bound for $\bpit$ we require that:
\begin{equation}
    2\sqrt{\frac{C(\delta)}{N(1-\gamma)^3}} \;\leq\; \frac{\Delta}{5} \;\;\Rightarrow\;\; N \,\geq\, \cO \left(\frac{C(\delta)}{(1-\gamma)^3\Delta^2} \right) \,=\, \cOT \left( \frac{d^3 \log\left(\frac{1}{\delta} \right)}{(1-\gamma)^5\epsilon^2{\slatercs}^2} \right) \label{eq:strict-data-dep-sample-compl}
\end{equation}

To bound the remaining concentration terms for $\pis$ and $\pics$ in \cref{eq:strict-feasibility-requirements} we use \cref{lm:data-dep-policies-concentration-bounds} with $\tpi = \pis, \pics$. In that case for $C'(\delta) = 72 \log\left(\frac{4|S| \log(e/1-\gamma)}{\delta} \right)$, we require that:
\begin{align*}
    2\sqrt{\frac{C'(\delta)}{N (1 - \gamma)^3}} \;\leq\; \frac{\epsilon}{5};\hspace{10mm} \sqrt{\frac{C'\left( \frac{\delta}{d} \right)}{N (1 - \gamma)^3}} \;\leq\; \Delta;\hspace{10mm} \sqrt{\frac{C'\left( \frac{\delta}{d} \right)}{N (1 - \gamma)^3}} \;\leq\; \frac{19\Delta}{5} 
\end{align*}
\begin{align*}
    \hspace{2mm} \Rightarrow\;\, N \,\geq\, \cO \left(\frac{C'(\delta)}{(1-\gamma)^3\epsilon^2} \right);\hspace{5mm} N \,\geq\, \cO \left(\frac{C'\left( \frac{\delta}{d} \right)}{(1-\gamma)^3\Delta^2} \right);\hspace{8mm} N \geq \cO \left(\frac{C'\left( \frac{\delta}{d} \right)}{(1-\gamma)^3\Delta^2} \right) \hspace{12mm}
\end{align*}
\begin{align}
    \Rightarrow\;\, N \,\geq\, \cOT \left( \frac{d^2 \log \left(\frac{d}{\delta} \right)}{(1-\gamma)^5\epsilon^2{\slatercs}^2}\right) \label{eq:strict-data-ind-sample-compl}
\end{align}

Combining the results for the data-dependent policy $\bpit$ in \cref{eq:strict-data-dep-sample-compl} that hold with probability $1 - \delta/5$, and the results for the data-independent policy $\pis$ in \cref{eq:strict-data-ind-sample-compl} that hold with probability $1 - 3\delta$ (using a union bound), taking the loosest sample complexity for $N$, which is $\cOT \left( \frac{d^3 \log\left(\frac{1}{\delta} \right)}{(1-\gamma)^5\epsilon^2{\slatercs}^2} \right)$ and then using a union bound, we finally get that the bounds in \cref{eq:strict-feasibility-requirements} are satisfied with probability at least $1 - (3\delta + \delta/5) \,\geq\, 1 - 4\delta$, which completes the proof.
\end{proof}

\begin{lemma}[Decomposing the suboptimality]\label{lm:decomposing-suboptimality-strict}
    For a fixed $\Delta > 0$ and $\eopt < \Delta$, if $b_i' = b_i + \Delta,\; \fac$, if the following conditions are satisfied $\fac$:
    \begin{equation*}
        \left| \Vcibpi - \hVcibpi \right| \,\leq\, \Delta - \eopt;\;\, \left| \Vcipis - \hVcipis \right| \,\leq\, \Delta
    \end{equation*}
    then (a) the policy $\bpi_T$ satisfies the constraints i.e. $\Vcibpi \geq b_i\, \fac$ and (b) its optimality gap can be bounded as:
    \begin{equation*}
        \Vrpis - \Vrbpi \;\leq\; \frac{2\omega}{1-\gamma} + \eopt + 2d\lambda_{\min}^*\Delta + \left[ \Vrppis - \hVrppis \right] + \left[ \Vrpbpi - \Vrpbpi \right]
    \end{equation*}
\end{lemma}

\begin{proof}

In this case we define a slightly modified CMDP problem compared to the one defined in Eq. \ref{eq:emprical-CMDP-formulation} where we change the constraint RHS to $\bb''$ which will be specified later. The new optimal policy is defined as $\tpis$, or more specifically:

\begin{equation}
    \tpis \in \argmax\pi\, \hVrppi \quad \mathrm{s.t.}\quad \hVcipi \geq b_i'',\; \fac \label{eq:strict-CMDP}
\end{equation}

From \cref{th:theorem-1} we have that $\fac$:
\begin{align*}
    \hVcibpi &\,\geq\, b' - \eopt \;\Rightarrow \\[1ex]
    \Vcibpi &\,\geq\, \Vcibpi - \hVcibpi +  b' - \eopt \\[1ex]
    &\,\geq\, - \left|\Vcibpi - \hVcibpi \right| + b' - \eopt
\end{align*}

Given that for the policy $\bpit$ must hold $\Vcibpi \geq b_i\; \fac$ we get that:
\[
    \Vcibpi \,\geq\, - \left|\Vcibpi - \hVcibpi \right| + b' - \eopt \geq b_i
\]

Thus we require that:
\[
    \left|\Vcibpi - \hVcibpi \right| \,\leq\, (b' - b) - \eopt
\]

To continue our analysis we will also require that $\pis$ is a feasible policy for the constraint problem defined in \cref{eq:strict-CMDP}. In more detail we require $\Vcipis \geq b_i'',\; \fac$. Given that $\pis$ is the solution to \cref{eq:CMDP-formulation} we get that $\fac$:
\[
    \Vcipis \geq b_i \;\Rightarrow\; \hVcipis \geq b_i - \left| \Vcipis - \hVcipis \right|
\]

Since we require $\Vcipis \geq b_i'',\; \fac$ a sufficient condition is:
\[
    b_i - \left| \Vcipis - \hVcipis \right| \geq b_i'' \;\Rightarrow\; \left| \Vcipis - \hVcipis \right| \leq b_i - b_i''
\]

Summarizing we require the below to statements to hold
\[
    \left| \Vcibpi - \hVcibpi \right| \,\leq\, (b_i' - b_i) - \eopt;\quad \left|\Vcipis - \hVcipis \right| \leq b_i - b_i''
\]

If the above to statements hold then, we can decompose the suboptimality in the reward values function as:
\begin{align*}
    \Vrpis\, -&\;\Vrbpi \;= \\[1ex]
        =&\; \Vrpis - \Vrppis + \Vrppis - \Vrbpi \tag{$\pm \; \Vrppis$} \\[1ex]
        =&\; \left[ \Vrpis - \Vrppis \right] + \left[ \Vrppis - \hVrppis \right] + \hVrppis - \Vrbpi \tag{$\pm \; \hVrppis$} \\[1ex]
        \leq&\; \left[ \Vrpis - \Vrppis \right] + \left[ \Vrppis - \hVrppis \right] + \hVrptpis - \Vrbpi \tag{$\tpis$ maximizes $\hVrp$ across feasible policies, and showed that $\pi^*$ is feasible} \\[1ex]
        =&\; \left[ \Vrpis - \Vrppis \right] + \left[ \Vrppis - \hVrppis \right] + \left[ \hVrptpis - \hVs \right] + \hVs - \Vrbpi \tag{$\pm \; \hVs$} \\[1ex]
        =&\; \left[ \Vrpis - \Vrppis \right] + \left[ \Vrppis - \hVrppis \right] + \left[ \hVrptpis - \hVs \right] + \left[ \hVs - \hVrpbpi \right] \\[1ex]
        &\qquad\qquad\qquad\qquad\;\;\, + \hVrpbpi - \Vrbpi \tag{$\pm \; \hVrpbpi$} \\[1ex]
        =&\; \underbrace{\left[ \Vrpis - \Vrppis \right]}_{\text{Perturbation Error}} + \underbrace{\left[ \Vrppis - \hVrppis \right]}_{\text{Concentration Error}} + \underbrace{\left[ \hVrptpis - \hVs \right]}_{\text{Sensitivity Error}} + \underbrace{\left[ \hVs - \hVrpbpi \right]}_{\text{Primal-Dual Error}} \\[1ex]
        &\qquad\qquad\qquad\qquad\;\;\, + \underbrace{\left[ \hVrpbpi - \Vrpbpi \right]}_{\text{Concentration Error}} + \underbrace{\left[ \Vrpbpi - \Vrbpi \right]}_{\text{Perturbation Error}} \tag{$\pm \; \Vrpbpi$}
\end{align*}

To bound the perturbation errors we use \cref{lm:diff-of-perturbed-value-functions-bound} while to bound the primal-dual error by $\eopt$ we use \cref{th:theorem-1}. Thus, we end up with, 

\[
    \Vrpis\, - \;\Vrbpi \leq \frac{2\omega}{1-\gamma} + \eopt + \underbrace{\left[ \Vrppis - \hVrppis  \right]}_{\text{Concentration Error}} +  \underbrace{\left[ \hVrptpis - \hVs \right]}_{\text{Sensitivity Error}} + \underbrace{\left[ \hVrpbpi - \Vrpbpi \right]}_{\text{Concentration Error}}
\]

Since $b_i' = b_i + \Delta,\; \fac$ and setting $b_i'' = b_i - \Delta$ we use \cref{lm:bounding_sensitivy_error} to bound the error sensitivity term to get:
\[
    \hspace{-7mm} \Vrpis - \Vrbpi \;\leq\; \underbrace{\frac{2\omega}{1-\gamma}}_{\text{Perturbation Errors}} + \underbrace{\eopt}_{\text{Primal-Dual Error}} + \underbrace{2d\lambda_{\min}^*\Delta}_{\text{Sensitivity Error}} + \underbrace{\left[ \Vrppis - \hVrppis \right]}_{\text{Concentration Error}} + \underbrace{\left[ \hVrpbpi - \Vrpbpi \right]}_{\text{Concentration Error}}
\]
\end{proof}

\begin{lemma}[Bounding the sensitivity error]\label{lm:bounding_sensitivy_error}
    If $b_i' = b_i + \Delta$ and $b_i'' = b_i - \Delta,\; \fac$ in \cref{eq:emprical-CMDP-formulation} and \cref{eq:strict-CMDP-formulation} such that $\fac$:
    \begin{align*}
         &\hpis \in \argmax_\pi \hVrppi \text{ s.t. } \hVcipi \geq b_i + \Delta \\
         &\tpis \in \argmax_\pi \hVrppi \text{ s.t. } \hVcipi \geq b_i - \Delta
    \end{align*}
    then the sensitivity error term can be bounded by:
    \[
        \left| \hVrptpis - \hVs \right| \,\leq\, 2d \lambda_{\min}^* \Delta
    \]
\end{lemma}

\begin{proof}
From the Lagrangian form of the empirical CMDP in \cref{eq:emprical-CMDP-formulation} we have 
\begin{align*}
    \hVs =&\; \max_\pi \min_{\lambda_i \geq 0} \hVrppi + \bltop [\hVcpi - (\bb + \Delta)] \\[1ex]
    =&\; \min_{\lambda_i \geq 0} \max_\pi \hVrppi + \bltop [\hVcpi - (\bb + \Delta)] \tag{\text{strong duality Lemma \ref{lm:dual-variable-bound}}} \\[1ex]
    =&\; \max_\pi \hVrppi + \blstop [\hVcpi - (\bb + \Delta)] \tag{$\bls$ optimal dual variable of Eq. \ref{eq:emprical-CMDP-formulation}} \\[1ex]
    \geq&\; \hVrptpis + \blstop [\hVctpis - (\bb + \Delta)] \tag{since the maximization is over all policies}
\end{align*}

Since $\fac:\;\; \hVcitpis \geq b_i - \Delta \;\,\Leftrightarrow\,\; \hVcitpis - (b_i + \Delta) \geq -2\Delta$, we have
\[
    \hVs \,\geq\, \hVrptpis - 2d \lambda_{\min}^* \Delta
\]

Therefore
\[
    \hVrptpis - \hVs \,\leq\, 2d \lambda_{\min}^* \Delta
\]

Considering that the RHS of \cref{eq:strict-CMDP-formulation} ($\bb'' = \bb - \Delta$) defines a less constrained problem compared to \cref{eq:emprical-CMDP-formulation} ($\bb' = \bb + \Delta$) we get that $\hVrptpis \geq \hVs$ and hence,
\[
    \left| \hVrptpis - \hVs \right| \,\leq\, 2d \lambda_{\min}^* \Delta
\]
\end{proof}